\documentclass[twoside]{aiml16}
\usepackage{aiml16macro}

\usepackage[UKenglish]{babel}
\usepackage{latexsym}
\usepackage{amssymb}
\usepackage{amsmath}
\usepackage{etex}
\usepackage{color}
\usepackage[all]{xy}
\usepackage{url}
\usepackage[bitstream-charter]{mathdesign}
\usepackage[T1]{fontenc}

\newcommand{\red}[1]{\textcolor{red}{#1}}
\renewcommand{\@}{\red{@}}

\newcommand{\ELKv}{\textbf{ELKv}}
\newcommand{\ELKvr}{\ensuremath{\ELKv^r}}

\newcommand{\LKv}{\textbf{MLKv}}
\newcommand{\MLKv}{\textbf{MLKv}}
\newcommand{\LKvr}{\ensuremath{\LKv^r}}

\newcommand{\MLKvb}{\ensuremath{{\LKv^b}}}

\newcommand{\MLKvr}{\ensuremath{\textbf{MLKv}^r}}

\newcommand{\BP}{\textbf{P}}

\newcommand{\SMLKv}{\ensuremath{\mathbb{SMLKV}}}

\newcommand{\SELKvr}{\ensuremath{\mathbb{SELKV}^r}}
\newcommand{\SMLKvr}{\ensuremath{\mathbb{SMLKV}^r}}

\newcommand{\SMLKvb}{\ensuremath{{\mathbb{SMLKV}^b}}}

\newcommand{\lr}[1]{\langle #1 \rangle}
\newcommand{\rel}[1]{\stackrel{#1}{\rightarrow}}

\newcommand{\mc}[1]{\mathcal{#1}}
\newcommand{\M}{\mc{M}}

\newcommand{\tr}[1]{\text{#1}}
\newcommand{\natN}{\mathbb{N}}

\newcommand{\I}{\mathbf{I}}
\newcommand{\Dia}{\Diamond}

\newcommand{\DISTK}{{\texttt{DISTK}}}
\newcommand{\DISBK}{{\texttt{ATEUC}}}
\newcommand{\DISTKvb}{${\texttt{DISTKv}}^b$}

\newcommand{\GENK}{{\texttt{NECK}}}
\newcommand{\EXIST}{\texttt{INCL}}
\newcommand{\NECKvr}{${\texttt{NECKv}}^r$}
\newcommand{\NECKvb}{${\texttt{NECKv}}^b$}

\newcommand{\RE}{{\texttt{RE}}}

\newcommand{\SUB}{{\texttt{SUB}}}

\newcommand{\AxT}{{\texttt{T}}}
\newcommand{\AxTr}{{\texttt{4}}}
\newcommand{\AxEuc}{{\texttt{5}}}

\newcommand{\DISTKvr}{\ensuremath{\texttt{DISTKv}^r}}
\newcommand{\KvrTr}{\ensuremath{\texttt{Kv}^r4}}
\newcommand{\KvrDIS}{\ensuremath{\texttt{Kv}^r\lor}}
\newcommand{\KvrBot}{\ensuremath{\texttt{Kv}^r\bot}}

\newcommand{\KvDIS}{\ensuremath{\texttt{Kv}\lor}}

\newcommand{\EQUIV}{\ensuremath\texttt{SYM}}

\newcommand{\N}{\mathcal{N}}

\newcommand{\Kv}{\ensuremath{\textit{Kv}}}
\newcommand{\Kw}{\ensuremath{\textit{Kw}}}
\newcommand{\Kh}{\ensuremath{\textit{Kh}}}

\newcommand{\K}{\ensuremath{\textit{K}}}
\newcommand{\hK}{\ensuremath{\hat{\textit{K}}}}
\newcommand{\C}{\mathbf{C}}
\newcommand{\lra}{\ensuremath{\leftrightarrow}}
\newcommand{\Lra}{\ensuremath{\Leftrightarrow}}

\newcommand{\Z}{\texttt{Z}}

\newcommand{\bis}{\mathrel{\mathchoice%
{\raisebox{.3ex}{$\,
  \underline{\makebox[.7em]{$\leftrightarrow$}}\,$}}%
{\raisebox{.3ex}{$\,
  \underline{\makebox[.7em]{$\leftrightarrow$}}\,$}}%
{\raisebox{.2ex}{$\,
  \underline{\makebox[.5em]{\scriptsize$\leftrightarrow$}}\,$}}%
{\raisebox{.2ex}{$\,
  \underline{\makebox[.5em]{\scriptsize$\leftrightarrow$}}\,$}}}}
  
\renewenvironment{definition}{\vspace{-\lastskip}\par \addvspace{\thmvspace}\begin{defn}}{\end{defn}\par\addvspace{\thmvspace}}

\newcommand{\cbis}{\bis_\textbf{C}}
\renewcommand{\lra}{\leftrightarrow}
\newcommand{\bBox}[2]{{\Box_{#1}^#2}}
\newcommand{\bDia}[2]{{{\Diamond}_{#1}^#2}}
\newcommand{\bbBox}[2]{{\Box_{#1}^#2}}
\newcommand{\bbDia}[2]{{\Diamond_{#1}^#2}}

\newcommand{\To}{\hookrightarrow}

\newcommand{\VV}{\overline{V}}
\newcommand{\UU}{\overline{U}}

\def\lastname{ }

\begin{document}
\begin{frontmatter}
 \title{``Knowing value'' logic as a normal modal logic }
 \thanks{Published in the proceedings of AiML 2016 pp. 362-381 by College Publications. This is a draft with a more detailed proof of Proposition \ref{prop.consis}.}
\thanks{Yanjing Wang acknowledges the support from the National Program for Special Support of Eminent Professionals and NSSF key projects 12\&ZD119. The authors thank the anonymous reviewers for their helpful comments on an earlier version of this paper.}
\author{Tao Gu \and Yanjing Wang}
 \address{Department of Philosophy, Peking University}
\maketitle

\begin{abstract}
Recent years witness a growing interest in nonstandard epistemic logics of ``knowing whether'', ``knowing what'', ``knowing how'', and so on. These logics are usually not normal, i.e., the standard axioms and reasoning rules for modal logic may be invalid. In this paper, we show that the conditional ``knowing value'' logic proposed by Wang and Fan \cite{WF13} can be viewed as a disguised normal modal logic by treating the negation of the Kv operator as a special diamond. Under this perspective, it turns out that the original first-order Kripke semantics can be greatly simplified by introducing a ternary relation $R_i^c$ in standard Kripke models, which associates one world with two $i$-accessible worlds that do not agree on the value of constant $c$. Under intuitive constraints, the modal logic based on such Kripke models is exactly the one studied by Wang and Fan \cite{WF13,WangF14}. Moreover, there is a very natural binary generalization of the ``knowing value'' diamond, which, surprisingly, does not increase the expressive power of the logic. The resulting logic with the binary diamond has a transparent normal modal system, which sharpens our understanding of the ``knowing value'' logic and simplifies some previously hard problems.   
\end{abstract}

\end{frontmatter}

% \noteYW{
% To be removed after completion: 
% \begin{itemize}
% \item Prove that the new system is equivalent to \ELKvr
% \item Give the alternative semantics
% \item Prove the completeness
% \item Give a decision procedure and prove the computational complexity
% \item Check its relation with other non-standard epistemic logic.
% \item This idea may help in axiomatizing PALKv (without the conditional operator)
% \end{itemize}}

\section{Introduction}
%\noteYW{The points to be mentioned
%\begin{itemize}
%\item Nonstandard epistemic logics are getting more popular
%\item They are mostly not normal.
%\item Yet we do not use more general semantics
%\item The project here is to look at the knowing value logic from a more traditional point of view
%\item The standard picture: language, logic and semantics
%\item The more general picture is to simplify the semantics by preserving the same logic. 
%\item We can then treat the new logics using variations of the standard tools. 
%\end{itemize}
%}

Classic epistemic logic \`{a} la von Wright and Hintikka mainly studies the inference patterns about propositional knowledge by using a modal operator $K_i$ to express that agent $i$ \textit{knows that} a proposition is true. Epistemic logic has been successfully applied to various fields to capture knowledge and its change in multi-agent settings, such as distributed systems and imperfect information games (cf. e.g. \cite{RAK,ELbook}). However, in everyday life, knowledge is often expressed in terms of knowing the answer to an embedding question, such as ``I know \textit{whether} the claim is true'', ``I know \textit{what} your password is'', ``I know \textit{how} to prove the theorem'' and so on. Recent years witness a growing interest in the logics of such knowledge expressions \cite{Plaza89:lopc,hoeketal:2004,WF13,WangF14,FanWD14,FWvD15,Wang15:lori}. The fundamental idea is to simply treat ``knowing whether'', ``knowing what'', ``knowing how'' as new modalities, just as ``knowing that'' in standard epistemic logic (cf. the survey \cite{BKT}).

The resulting logics are usually not \textit{normal} in the technical sense that usual modal axioms and rules may be invalid. For example, the \texttt{K} axiom for normal modal logic is not valid for the knowing whether operator, i.e., $\Kw(p\to q)\land \Kw p\to \Kw q$ does not hold, e.g., knowing that $p$ is false makes sure that you know whether $p$ and also whether $p\to q$, but it does not tell you anything about the truth value of $q$. Similarly, knowing how to swim and knowing how to cook does not mean knowing how to swim and cook at the same time, thus invalidating $\Kh p\land \Kh q\to \Kh(p\land q)$, a theorem in normal modal logic when taking $\Kh$ as a box modality. 

On the other hand, the non-normality does not necessarily mean that we have to abandon Kripke models for more general models. As demonstrated in \cite{FWvD15}, we can still use Kripke models to accommodate those non-normal modal logics by using nonstandard yet intuitive truth conditions for the new modalities. However, there is usually a clear asymmetry between the relatively simple modal language and the ``rich'' model which may cause troubles in axiomatizing the logic. For example, the conditional knowing value logic proposed in  \cite{WF13} has the following language \ELKvr (where $i\in \I, p\in \BP, c\in \C$ and $\I, \BP,\C$ are countably infinite):
$$\phi::=\top\mid p\mid \neg\phi\mid (\phi\wedge\phi)\mid \K_i\phi\mid \Kv_i(\phi,c)$$
\noindent where $\Kv_i(\phi, c)$ says that $i$ knows [what] the value of $c$ [is], given $\phi$, e.g., I know the password of this website given it is 4-digit, since I may have only one 4-digit password ever, although I am not sure which password I used for this website without the information on the digits. The language is interpreted on first-order Kripke models $\M=\langle S,D,\{\to_i:i\in \I\}, V, V_\C\rangle$ where $\langle S, \{\to_i:i\in \I\}, V\rangle$ is a standard Kripke model, and $D$ is a \textit{constant} domain, and $V_\C$ assigns to each (non-rigid) $c\in \C$ an element in $D$ on each $s\in S$. The semantics for the new $\Kv_i$ operator is as follows: 
$$
\begin{array}{|lll|}
\hline
\M,s\vDash \Kv_i (\phi, c) &\Longleftrightarrow& \text{for any}~t_1,t_2: ~\text{if}~s\to_it_1,s\to_it_2, \M, t_1\vDash \phi \text{ and } \M ,t_2\vDash \phi,\\ &&\text{then}~V_\C(c,t_1)=V_\C(c,t_2).\\
\hline
\end{array}
$$
According to this semantics, the formula $\Kv_i(\phi,c)$ can also be understood as a first-order modal formula: $\exists x K_i(\phi\to c=x)$.\footnote{Note that there is a constant domain $D$ and each $c$ has a unique value on each state.} Thus $\ELKvr$ can be viewed as a (small) fragment of first-order modal logic where a quantifier is packed with a modality. It is shown in \cite{WF13} that \ELKvr\ is equally expressive as public announcement logic extended with unconditional $\Kv_i$ operators proposed in \cite{Plaza89:lopc} (i.e., only $\Kv_i(\top, c)$ are allowed). Satisfiability of $\ELKvr$ over arbitrary models is  \textsc{Pspace}-complete, as proved in \cite{Ding15}. Note that although values are assigned to the constants in the model, we cannot talk about them explicitly in the language. In fact, we only care about whether on some worlds a given constant has exactly the \textit{same} value. The contrast between the rich model and the simple language made the completeness proof of the following axiomatization $\SELKvr\mathbb{S}5$ quite involved over multi-agent S5 models (cf. \cite{WangF14}).\footnote{$\phi[\psi\slash\chi]$ in the rule $\RE$ (replacement of equivalents) denotes any formula obtained by replacing \textit{some} occurrences of $\psi$ by $\chi$.} 
\begin{center}
\begin{minipage}{0.55\textwidth}
\begin{tabular}{lc}
\multicolumn{2}{c}{System $\SELKvr\mathbb{S}5$}\\
% \lcline{1-2}\rcline{3-4}
\multicolumn{2}{l}{Axiom Schemas}\\
%\lcline{1-2}\rcline{3-4}
\texttt{TAUT} & \tr{ all the instances of tautologies}\\
\DISTK &$ \K_i (p\to q)\to (\K_i p\to \K_i q$)\\
%\texttt{T} &$ \K_i p\to p$)\\\textbf{ }
\AxT&$\K_ip\to p$ \\
\AxTr&$\K_ip\to \K_i\K_ip$ \\
\AxEuc&$\neg \K_ip\to \K_i\neg \K_ip$ \\
\DISTKvr& $\K_i (p\to q)\to (\Kv_i(q,c)\to \Kv_i(p,c))$\\
\KvrTr& $\Kv_i(p,c)\to \K_i \Kv_i(p,c)$\\
\KvrBot& $\Kv_i(\bot,c)$\\
\KvrDIS& $\hK_i(p\wedge q)\wedge \Kv_i(p,c)\wedge \Kv_i(q,c)\to \Kv_i(p\vee q,c)$\\
\end{tabular}
\end{minipage}
\hfill
\begin{minipage}{0.3\textwidth}
\begin{tabular}{lc}
\multicolumn{2}{l}{Rules}\\
\texttt{MP}& $\dfrac{p,p\to q}{q}$\\
\GENK &$\dfrac{\phi}{\K_i \phi}$\\
\SUB &$\dfrac{\phi}{\phi [p\slash\psi]}$\\
\RE &$\dfrac{\psi\lra\chi}{\phi\lra \phi[\psi\slash\chi]}$
\end{tabular}
\end{minipage}
\end{center}
Since the $\Kv_i$ operator is not a modality taking only propositions as  arguments, it is hard to say whether the above logic is normal or not. $\DISTKvr$ looks a little bit like the \texttt{K} axiom but it is in fact about the interaction between $\K_i$ and $\Kv_i$. $\KvrTr$ is a variation of the positive introspection axiom, and the corresponding negative introspection is derivable. \KvrBot\ says that the $\Kv_i$ operator is essentially a  conditional. Axiom $\KvrDIS$ handles the composition of the conditions, where $\hK_i$ is the dual of $\K_i$. 

\medskip

In this paper, we look at $\ELKvr$ from a new yet ``normal modal logic'' perspective in order to answer the following questions: 
\begin{enumerate}
\item Since we do not talk about values in the language, is there a simpler value-free Kripke-model based semantics for $\ELKvr$ that can keep the logic (valid formulas) the same? If so, we can restore the symmetry between the language and the model and understand the essence of our logic.
\item Can $\ELKvr$ be linked to a normal modal logic (modulo some syntactic transformation)? If so, we can apply many standard modal logic techniques to simplify previously complicated discussions. 
\end{enumerate}

We give positive answers to both questions, inspired by a crucial observation: 
\paragraph{Observation} \textit{$\neg \Kv_i(\phi, c)$ can be viewed as a diamond operator $\bDia{i}{c}\phi$ which says that there are two $i$-accessible $\phi$-worlds, which do not agree on the value of $c$.}\\

Note that to simplify the technical discussion in order to reveal the crucial points, in this paper we focus on the logic over arbitrary models. Our techniques can be applied to the S5 setting. 
% At present we have only discussed the corresponding non-$S5$ system $\LKvr$. The related work of $S5$ system $\ELKvr$ is coming soon.

% \noteYW{$\bDia{i}{c}(p\lor q)\to  (\bDia{i}{c}p\lor \bDia{i}{c}q)$ is not valid so the logic is not normal but it is a misleading fact if we restore the more natural syntax keeping the expressivity the same we will discover a normal modal logic underneath \ELKvr. }
\medskip

The contributions of this paper are summarized as below: 
\begin{itemize}
\item We give a simple alternative Kripke semantics to $\ELKvr$ without value assignments, which does not change the set of valid formulas. The  completeness proof is much simpler compared to the one in \cite{WangF14}.    
\item We generalize $\bDia{i}{c}\phi$ in a natural way to a binary diamond operator $\bDia{i}{c}(\phi,\psi)$. It turns out the generalization does not increase the expressive power of the logic but it can give us a transparent normal modal logic proof system.
\item The normal modal logic perspective helps us to discover a bisimulation notion for $\ELKvr$ and obtain a proof system for a weaker language proposed by \cite{Plaza89:lopc}. 
\end{itemize}
Our findings show that $\ELKvr$ is essentially a ``disguised'' normal logic, and this may help us to understand such nonstandard epistemic operators better.
%since they roughly share the same defining schema: $\exists x\Box\psi(x)$.

\medskip

The rest of the paper is organized as follows: we first introduce in Section~\ref{Sec.intro} the language with the unary diamond $\bDia{i}{c}$ and a semantics based on Kripke model with both binary and ternary relations under three intuitive constraints. We show that this semantics is  equivalent to the original FO Kripke semantics of \ELKvr\ modulo validity (under a straightforward syntactic translation). In Section~\ref{Sec.comp} we prove the completeness of the translated $\SELKvr$ system w.r.t.\ the new semantics directly. This demonstrates the advantages of using this simplified semantics. In Section~\ref{Sec.gen} we generalize the $\bDia{i}{c}$ naturally to a binary one and show that the extended language is in fact equally expressive as $\ELKvr$. On the other hand,  the extended language facilitates a transparent normal logic proof system. It then helps us in Section 5 to come up with a notion of bisimulation for $\ELKvr$ and obtain a proof system for a weaker language proposed earlier. We conclude  in the end with future directions. 
%\noteYW{Mention the works on Atoms. }

\section{Negation of $\Kv_i$ as a diamond}
\label{Sec.intro}
As we mentioned in the introduction, $\neg\Kv_i(\phi, c)$ can be viewed as a diamond formula $\bDia{i}{c}\phi$: there are two $i$-accessible $\phi$-worlds which do not agree on the value of $c$. Then $\bBox{i}{c}\phi:=\neg \bDia{i}{c}\neg \phi$ means that all the $i$-accessible $\neg \phi$-worlds agree on the value of $c$ (and it is $\Kv_i(\neg \phi,c)$ essentially.). For uniformity of the language, we take $\bBox{i}{c}$ as the primitive symbol and introduce the following language ($\LKvr$): 

$$\phi::=\top\mid p\mid \neg\phi\mid (\phi\wedge\phi)\mid \Box_i\phi\mid \bBox{i}{c}\phi$$
%where $\bBox{i}{c}\phi$ is essentially a binary modality where the two arguments are the same. 
\noindent where $p\in\BP, c\in \C, i\in \I.$ We can, without difficulty, inductively define a translation function $T$ from the original $\ELKvr$ to this modal language (the other way is also straightforward):

\begin{definition}
A translation function $T$ from $\ELKvr$  to \LKvr formulas is defined as follows:
$$\begin{array}{rcl}
T(p)&=&p \\
T(\neg\phi)&=&\neg T(\phi)\\
T(\phi\land\psi)&=&T(\phi)\land T(\psi)\\
T(\K_i\phi)&=&\Box_i T(\phi)\\
T(\Kv_i(\phi,c))&=&\bBox{i}{c}\neg T(\phi)
\end{array}
$$ 
\end{definition}
%%直接翻译原来的系统
Now we have the following translated axioms (the names are kept): 
$$\begin{array}{rcl}
T(\DISTKvr)&=& \Box_i(p\to q)\to (\bBox{i}{c}\neg q\to \bBox{i}{c}\neg p) \\
T(\KvrDIS)&=&\Diamond_i(p\land q)\land \bBox{i}{c}\neg p\land \bBox{i}{c}\neg q\to \bBox{i}{c}\neg(p\lor q)\\
T(\KvrBot)&=&\bBox{i}{c}(\neg \bot)
\end{array}
$$ 
We can massage the axioms,\footnote{For example, T(\DISTKvr) is equivalent to $\Box_i(\neg q\to \neg p)\to (\bBox{i}{c}\neg q\to \bBox{i}{c}\neg p)$ under \RE, which is equivalent to $\Box_i (p\to q)\to (\Box_i p\to \Box_i q)$ under \SUB.}and obtain the following equivalent system (modulo translation T) from \SELKvr\ (i.e., \SELKvr-$\mathbb{S}5$ without the S5 related axioms $\AxT, \AxTr, \AxEuc, \KvrTr $):\\
\begin{minipage}{0.64\textwidth}
\begin{center}
\begin{tabular}{lc}
\multicolumn{2}{c}{System \SMLKvr}\\
% \lcline{1-2}\rcline{3-4}
\multicolumn{2}{l}{Axiom Schemas}\\
%\lcline{1-2}\rcline{3-4}
\texttt{TAUT} & \tr{ all the instances of tautologies}\\
\DISTK &$ \Box_i (p\to q)\to (\Box_i p\to \Box_i q)$\\
$\DISTKvr$ &$ \Box_i(p\to q)\to (\bBox{i}{c} p\to \bBox{i}{c}q$)\\
% $\KvrTr$ &$\bBox{i}{c}\phi \to \Box_i\bBox{i}{c}\phi$\\
$\KvrDIS$ &$\Dia_i (p\land q)\land \bDia{i}{c}(p\lor q)\to  (\bDia{i}{c}p\lor \bDia{i}{c}q)$\\
%\texttt{T} &$ \K_i \phi\to \phi$)\\\textbf{ }
\end{tabular}
\end{center}
\end{minipage}
\hfill
\begin{minipage}{0.30\textwidth}
\begin{center}
\begin{tabular}{lc}
\multicolumn{2}{l}{Rules}\\
\texttt{MP}& $\dfrac{\phi,\phi\to\psi}{\psi}$\\
\GENK &$\dfrac{\phi}{\Box_i \phi}$\\
\NECKvr &$\dfrac{\phi}{\bBox{i}{c}\phi}$\\
\SUB &$\dfrac{\phi}{\phi [p\slash\psi]}$\\
\RE &$\dfrac{\psi\lra\chi}{\phi\lra \phi[\psi\slash\chi]}$\\
%\GENAC &$\dfrac{\phi}{[c]\phi}$\\
\end{tabular}
\end{center}
\end{minipage}
\medskip

Note that instead of $\KvrBot$, we have a more classic-looking rule \NECKvr. In fact, it is equivalent to have either $\KvrBot$ or \NECKvr\ in the system: from \NECKvr, it is trivial to derive $\KvrBot$, and from $\KvrBot$, $\DISTKvr$ and \GENK\ it is also straightforward to derive each instance of \NECKvr\ by taking $p$ in $\DISTKvr$ as $\top.$ 

To a modal logician, $\SMLKvr$ may look much more friendly compared to $\SELKvr$. In particular, $\KvrDIS$ is simply a conditional distribution axiom for $\bDia{i}{c}$ over disjunction. Note that $\Diamond_i(p\lor q)\to (\Diamond_i p\lor \Diamond_i q)$ is valid but $\bDia{i}{c}(p\lor q)\to  (\bDia{i}{c}p\lor \bDia{i}{c}q)$ is not, e.g., all the $p$ worlds agree on the value of $c$ and all the $q$ worlds agree on the value of $c$ but they just cannot agree with each other. This demonstrates that $\bDia{i}{c}$ is apparently not a normal modality. However, as we will discover later that this apparent non-normality is a bit misleading and we will restore the normality in the next section by considering a natural binary generalization of the $\bDia{i}{c}$ operator.    

\medskip 

Now we are going to give a simplified but equivalent semantics to \MLKvr\ such that the system $\SMLKvr$ is  sound and complete. The idea is to abandon the first-order Kripke model and use a rather standard Kripke model for propositional modal logics since much of the information in the FO Kripke model is not relevant for the language \MLKvr. 

\begin{definition}
A model for $\LKvr$ is a tuple $\langle S,\{\to_i:i\in \I\},\{R_{i}^{c}:i\in \I,c\in\C\},V\rangle$, where
\begin{itemize}
\item $\lr{S,\{\to_i:i\in \I\},V}$ is a standard Kripke model with binary relations.
\item For each $c\in\C$, $R_{i}^{c}$ is a triple relation over $S$ satisfying for any $s,t,u,v\in S$: 
\begin{enumerate}
\item \textit{SYM}: $sR_i^c tu \iff sR_i^c ut$
\item \textit{INCL}:
 $sR_{i}^{c}tu$ only if $s\to_i t$ and $s\to_i u$
\item \textit{ATEUC:} $sR_{i}^{c}tu$ and $s\to_i v$ imply that at least one of $sR_{i}^{c}tv$ and  $sR_{i} ^{c}uv$ holds
\end{enumerate}

\end{itemize}
\end{definition}

%\noteYW{perhaps (3) can be called antieuclidean property: from $\forall xyz (xR'y\land xR'z\to yR'z)$ we have $\forall xyz (\neg xRy\land \neg xRz\to \neg yRz)$, namely $\forall xyz (yRz\to xRy\lor xRz)$  }

% As we mentioned in the introduction, the exact value of each $c\in\C$ is not expressible in the language thus it is unnecessary to have the assignment function $V_\textbf{C}: \textbf{C}\times S\rightarrow D$. Instead we have a ternary relation  $R_{i}^{c}$.
Intuitively, $sR_{i}^{c}tu$ roughly means that $s$ can see two $i$-accessible worlds $t,u$ which do not agree on the value of $c$, although we do not have value assignments for $c$ in the model. Further conditions are imposed to let the ternary relation really capture what we want. (i) is a symmetry condition on the later two arguments of $R^c_i$.  Condition (ii) establishes the connection between the ternary and binary relations. The most crucial condition is (iii), an anti-euclidean property\footnote{Euclidean property says that $\forall x,y,z: xRy\land xRz\to yRz$. Taking $R'=\bar{R}$ we have $\forall x,y,z:  \neg xR'y\land \neg xR'z\to \neg yR'z$, i.e., $\forall x,y,z: yR'z \to xR'y\lor xR'z$. Our condition is inspired by this observation.} that says if two $i$-accessible worlds do not agree on the value of $c$ then for any third $i$-accessible world it must disagree with one of the two worlds on $c$.\footnote{Careful readers may wonder about whether we can break the the ternary relation into two: the $i$-relation and an anti-equivalence relation. We will come back to this point at the end of the paper.} %However, even if we start from the very intuition of differing w.r.t. values, the ternary relation $R_i^c$ satisfying the three conditions is not necessarily a .

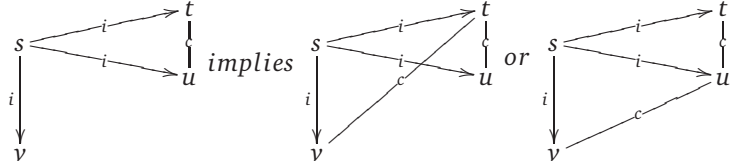
$$\xymatrix@R-20pt@C+30pt{
&t\ar@{-}[dd]|c\\
s \ar[ru]|i \ar[rd]|i\ar[dddd]_i&\\    
&u\\
&\\
&\\
v
}
\xymatrix@R-20pt@C+30pt{
\\
\\    
implies\\
\\
\\
}
\xymatrix@R-20pt@C+30pt{
&t\ar@{-}[dd]|c\ar@{-}[dddddl]|c\\
s \ar[ru]|i \ar[rd]|i\ar[dddd]_i&\\    
&u\\
&\\
&\\
v
}
\xymatrix@R-20pt@C+30pt{
\\
\\    
or \\
\\
\\
}\xymatrix@R-20pt@C+30pt{
&t\ar@{-}[dd]|c\\
s \ar[ru]|i \ar[rd]|i\ar[dddd]_i&\\    
&u\ar@{-}[dddl]|c\\
&\\
&\\
v
}
$$

% Finally, the states which are not connected with one given state doesn't need to be taken into account since we do not express their properties in our language, that's why we restrict $R_{i}^{c}$ to the $i$-successors of $s$. These observations would simplify our model.
The new semantics is defined as follows which reflects the intuition behind $R_{i}^{c}$:\\
$$\begin{array}{|lcl|}
\hline
\M,s\Vdash\top & & \text{ always }\\
\M,s\Vdash p &\Lra &s\in V(p)\\
\M,s\Vdash\neg\phi &\Lra & \M,s\not\Vdash\phi\\
\M,s\Vdash\phi\land\psi & \Lra&\M,s\Vdash\phi \text{ and } \M,s\Vdash\psi\\
\M,s\Vdash\Diamond_i\phi &\Lra& \text{ there exists $t$ such that $s\to_i t$ and }\M,t\Vdash\phi\\
\M,s\Vdash\bDia{i}{c}\phi & \Lra&\text{ there exist $u,v$ such that $sR_{i}^{c}uv$, $\M,u\Vdash\phi$ and $\M,v\Vdash\phi$}\\
\hline
\end{array}
$$

In the rest of this section, we show that the above semantics of \MLKvr\ is equivalent to the semantics for \ELKvr\ modulo the syntactic translation $T$. To show this, the difficult part is to saturate an  $\MLKvr$ model with value assignments while keeping the truth values of formulas modulo translation $T$. Note that this is not straightforward, as it is possible in an \MLKvr\ model that $sR_i^ctt$ and there is no way to assign to $c$ two different values on the same world $t$. Moreover, it can happen that $sR_i^cuv$, $s\to_ju$ and $s\to_jv$ while it is not the case that $sR_j^cuv$. However, we can avoid such problem by preprocessing the \MLKvr\ model before assigning valuations. 

\begin{lemma}\label{lem.trans}
For any set of $\ELKvr$ formula $\Sigma\cup\{\phi\}$, $\Sigma\vDash\phi$ iff $T(\Sigma)\Vdash T(\phi).$ 
\end{lemma}   
\begin{proof}
It suffices to prove that for any set of $\ELKvr$ formula $\Sigma$, $\Sigma$ is $\vDash$-satisfiable iff $T(\Sigma)$ is $\Vdash$-satisfiable. We say that an $\ELKvr$ model $\M,s$ is \textit{equivalent} to an $\MLKvr$ model $\N,t$ if for all $\phi\in\ELKvr$: $\M,s\vDash \phi\iff \N,t \Vdash T(\phi).$  In the following we show that for any pointed FO Kripke model $\M,s$ for \ELKvr, there is an equivalent $\MLKvr$ model $\N,t$, and vice versa. 

%%%%%%%%%%%%
$\Rightarrow$: For any pointed FO Kripke model $\M,s$, we can naturally define the ternary relation $R_i^c$ as follows: $sR_i^c tu$ iff $s\to_i t$, $s\to_i u$ and $V_C(c,t)\not=V_C(c,u)$. It's straightforward to check that the resulting model $\N$ is an $\LKvr$ model (satisfying the three conditions) and $\N,s$ is equivalent to $\M,s$.

%%%%%%%%%%%%%%%%
$\Leftarrow$: Recall that what we need to show is the following: given an $\LKvr$ model $\N=\lr{S,\{\to_i:i\in \I\},\{R_i^c:i\in\I,c\in \C\},V}$ and $t\in S$, find an \ELKvr\ model $\M,s$ such that $\M,s\vDash \phi\iff \N,t \Vdash T(\phi).$ As mentioned, we need to preprocess $\N$ before assigning values. 

The preprocessing consists of two steps: splitting and unraveling. We first split the states in $\N$ into two copies in order to handle the $uR_{i}^{c}vv$ problem mentioned before. Let $\N'$ be $\lr{S\times\{0,1\},\{\to'_i:i\in \I\},\{P_i^c:i\in\I,c\in \C\},V'}$, where:

\begin{itemize}
\item $(u,x)\to'_i (v,y) \iff u\to_i v$
\item $(u,x)P_i^c(v,y)(w,z) \iff uR_i^c vw$ and $(v,y)\not=(w,z)$
\item $V'((u,x))=V(u)$
\end{itemize}

It can be verified that $\N'$ has the three properties of \MLKvr\ models,\footnote{Take the anti-euclidean property as an example. Suppose $(u,x)P_i^c(v,y)(v',y')$ and $(u,x)\to_i'(w,z)$. If $(w,z)$ is one of $(v,y)$ and $(v',y')$, done. If not, then $(u,x)P_i^c(v,y)(v',y')$ implies $uR_i^c vw$, and $(u,x)\to_i'(w,z)$ implies $u\to_i w$. By the anti-euclidean property of the original model $\N$, we have either $uR_i^c vw$ or $uR_i^c v'w$. So either $(u,x)R_i^c (v,y)(w,z)$ or $(u,x)R_i^c (v',y')(w,z)$.} and there is no state $v$ such that $uP_i^c vv$ for any $u, v$. We can prove the following claim by a simple induction on the structure of $\MLKvr$ formulas: 
$$ \N,u\equiv_{\LKvr}\N',(u,x) \text{ where $x\in\{0,1\}$}$$
The only non-trivial case is when $\phi=\bDia{i}{c}\psi$. Suppose $\N,u\Vdash\bDia{i}{c}\psi$, then there exist $v,v'\in S$ ($v$ and $v'$ are not necessarily different) such that $uR_i^c vv'$, $\N,v\Vdash\psi$ and $\N,v'\Vdash\psi$. By the definition of $P_i^c$ and the induction hypothesis, $(u,x)P_i^c(v,0)(v',1)$, $\N',(v,0)\Vdash\psi$ and $\N',(v',1)\Vdash\psi$, so $\N',(u,x)\Vdash\bDia{i}{c}\psi$. Suppose $\N',(u,x)\Vdash\bDia{i}{c}\psi$, then there exist $(v,y)(v',z)$ such that $(u,x)P_i^c(v,y)(v',z)$, $\N,(v,y)\Vdash\psi$ and $\N,(v',z)\Vdash\psi$. According to the definition of $P_i^c$, this entails $(v,y)\not=(v',z)$ and $uR_i^c vv'$. By induction hypothesis, $\N,(v,y)\Vdash\psi$ and $\N,(v',z)\Vdash\psi$, so $\N,u\Vdash\bDia{i}{c}\psi$.  As a simple consequence, 
$$\N,s\equiv_{\LKvr}\N',(s,0)\eqno(1)$$

To simplify notation, we shall write $(s,0)$ as $s'$ in the rest of this proof.

% What's more, one can immediately verify that $\N,s \cbis \N',(s,0)$, by simply defining the binary relation $\Z$ to be: $s\Z(s,x)$. Therefore, according to the bisimulation result we have $\N,s\equiv_{\LKvr}\N',(s,0)$. (To simplify expression, we shall write $s'$ to represent $(s,0)$ in the following of this proof).

Now we unravel $\N'$ at $s'$ into $\M'=\lr{W,\{\To_i:i\in I\},\{Q_i^c:i\in\I,c\in\C\},U'}$:
\begin{itemize}
\item $W=\{\lr{s',i_1,v_1,\ldots,i_k,v_k}:$ there is a path $s'\rel{i_1} v_1\dots\rel{i_k} v_k$ in $\N'\}$. Note that the trivial path $\lr{s'}\in W$,
\item $\lr{s',i_1,\dots,v_k}\To_i \lr{s',j_1,\dots,u_m}$ iff $m=k+1, \lr{s',i_1,\dots,v_k}=\lr{s',j_1,\dots,u_k}, j_m=i$ and $v_k\to_i' u_m$ in $\N'$,
\item $\lr{s',i_1,\dots,v_k}Q_i^c \lr{s',j_1,\dots,u_m}\lr{s',l_1,\dots,l_n}$ iff $v_k P_i^c u_m l_n$, $\lr{s',i_1,\dots,v_k}\To_i \lr{s',j_1,\dots,u_m}$ and $\lr{s',i_1,\dots,v_k}\To_i\lr{s',l_1,\dots,l_n}$,
\item $U'(\lr{s',i_1,\dots,u})=V'(u)$.
\end{itemize}

Intuitively, the new model $\M'$ starts from $\lr{s'}$, and each state corresponds to a path which is accessible from $s'$ in $\N'$. It is not hard to verify the three properties of $\MLKvr$ models.\footnote{Again, take the anti-euclidean property as an example. Suppose $uQ_i^c vv'$, $u\To_i t$. Suppose $u=\lr{s',\dots,u_k}$, $v=\lr{s',\dots,v_m}$, $v'=\lr{s',\dots,v'_n}$, $w=\lr{s',\dots,w_l}$. Then $u_k P_i^c v_mv'_n$ and $u_k\to_i'w_l$, which implies at least one of $u_k P_i^c v_mw_l$ and $u_k P_i^c v'_nw_l$ holds. This together with $u\To_i v$, $u\To_i v'$ and $u\To_i w$ imply either $uQ_i^c vw$ or $uQ_i^c v'w$.} By definition, the $\To$ skeleton of $\M'$ is a tree-like structure: acyclic, every state except the root $\lr{s'}$ can be reached eventually by $\lr{s'}$ and has one and only one predecessor. It follows that for any $u,v$ in $\M'$, it is not the case that $u\To_iv$ and $u\To_jv$ for any $i\not=j$.

Now we can prove the following by induction on the structure of $\phi\in\MLKvr$: $\M',\lr{s',i_1,v_1,\ldots,i_k,v_k}\equiv_{\MLKvr} \N',v_k$
In particular, $$\N',s'\equiv_{\MLKvr}\M',\lr{s'}\eqno(2)$$

Now the only thing left is to transform the $\LKvr$ model $\M'$ into an equivalent \ELKvr\ model $\M$. Basically, we just need to give values to $c\in\C$ on each state according to the ternary relations $Q_j^c$. Let $\M=\lr{W,D,\{\To_i:i\in\I\},U,V_\C}$ where:
\begin{itemize}
\item $W$ and $\{\To_i:i\in\I\}$ are exactly the same as in $\M'$;
\item $U=U'$;
\item $V_\C(c,w)=|(c,w)|_\sim$. That is, $V_\C(c,w)$ is the equivalence class under the equivalence relation $\sim$ over $\C\times W$ defined as: \\
${\sim}=\{\lr{(c,u),(e,v)}:c=e, \exists s \exists j:  s\To_j u, s\To_j v, \forall w\in W: \neg wQ_j^c uv \} \cup\{\lr{(c,u),(c,u)}\mid (c,u)\in \C\times W\}$

\item  $D=\{|(c,w)|_\sim\mid (c,w)\in \C\times W\}$;
\end{itemize}

To make sure $\M$ is well-defined, we need to show that $\sim$ is an equivalence relation. Reflexivity and symmetry are obvious, and for transitivity: 
%\begin{itemize}
%  \item \textit{Reflexivity}: By definition;
%  \item \textit{Symmetry}: $wQ_i^c tt' \iff wQ_i^c t't$;
If $(c,w)\sim(d,u)$, $(d,u)\sim(e,v)$, then $c=d=e$, there exist $s$, $i$ such that $s\To_i w$, $s\To_i u$ while for any $t$ not $tQ_i^cwu$, and there exist $s'$, $j$ such that $s'\To_j u$, $s'\To_j v$ while for any $t$ not $tQ_j^c uv$. Since every state in $W$ has at most one predecessor, $s=s'$. Since there is at most one relation between two different states, $i=j$. Therefore $s\To_i w$, $s\To_i v$ and $s\To_i u$. Suppose towards contradiction that there exists $o\in W$ such that $oQ_i^c wv$, then $o=s$. Thus $s Q_i^c wu$ or $s Q_i^c uv$ by anti-euclidean property, contradiction. Therefore $(c,w)\sim(e,v)$. 
%\end{itemize}

We still need to verify that this assignment is good, in the sense that: for any \ELKvr\ formula $\phi$, $\M',w\Vdash T(\phi) \iff \M,w\vDash \phi$ for any $w\in W$. We prove this by induction on $\phi$ and only show the non-trivial case:
%\begin{enumerate}
% \item If $\phi=p$, then $T(\phi)=p$. Since $U(w)=U'(w)$, done;

% \item If $\phi=\neg\psi$, then $T(\phi)=\neg T(\psi)$. By induction hypothesis, $\M,w\vDash T(\psi) \iff \M',w\Vdash\psi$, done;

% \item If $\phi=\psi\land\chi$, similar to the negation case;
% \item If $\phi=\hK_i\psi$, then $T(\phi)=\Dia_i T(\psi)$.

% $\Rightarrow$: Suppose $\M',w\Vdash\Dia_i T(\psi)$, then there exists $t$ such that $w\To_i t$ and $\M',t\Vdash T(\psi)$. By induction hypothesis, $\M,t\vDash \psi$, so $\M,w\vDash\Dia_i\psi$;

% $\Leftarrow$: Similarly.

If $\phi=\Kv_i(\psi,c)$, then $T(\phi)=\bBox{i}{c} \neg T(\psi)$.

$\Rightarrow$: Suppose $\M,w\not\vDash\Kv_i(\psi,c)$ then there exist $t,t'$ such that $w\To_i t$, $w\To_i t$, $\M,t\vDash\psi$, $\M,t'\vDash\psi$ and $(c,t)\not\sim(c,t')$. According to the definition of $\sim$, this implies $\exists u$ such that $uQ_i^c tt'$. But we have shown that every state has exactly one predecessor, so $u=w$, and $wQ_i^c tt'$. By induction hypothesis, $\M',t\Vdash T(\psi)$ and $\M',t'\Vdash T(\psi)$. Therefore, $\M',w\Vdash\bDia{i}{c} T(\psi)$, i.e., $\M',w\not\vDash\bBox{i}{c} \neg T(\psi)$.

$\Leftarrow$: Suppose $\M',w\not\Vdash\bBox{i}{c} \neg T(\psi)$, i.e. $\M',w\Vdash\bDia{i}{c} T(\psi)$.  Then there exist $t,t'\in W$ such that $wQ_i^c tt'$, $\M',t\Vdash T(\psi)$ and $\M',t'\Vdash T(\psi)$. So $w\To_i t, w\To_i t'$ but $(c,t)\not\sim(c,t')$, i.e. $V_\C(c,t)\not=V_\C(c,t')$. By induction hypothesis, $\M,t\vDash\psi$ and $\M,t'\vDash\psi$. Therefore, $\M,w\not\vDash\Kv_i(\psi,c)$.
%\end{enumerate} 

It follows that for any $\ELKvr$ formula $\phi$, 
$$\M',\lr{s'}\Vdash T(\phi)\iff \M,\lr{s'}\vDash\phi\eqno(3)$$

With (1), (2) and (3), we can now conclude that for any $\MLKvr$ model $\N,t$ there is always an equivalent $\ELKvr$ model $\M,s$ and this concludes the proof.
\end{proof}

\begin{remark}
The above lemma implies that for any $\ELKvr$ formula $\phi$: $$\vDash\phi \iff \Vdash T(\phi)$$ which asserts the validities are the same modulo the translation. We need the stronger version to handle strong completeness later.
\end{remark}

\section{Completeness of \SMLKvr}

\label{Sec.comp}
In this section, we show a direct proof of the strong completeness of \SMLKvr\ proposed in the previous section. As we will see, this proof is much simpler compared to the original completeness proof of $\SELKvr$ in \cite{WangF14} due to the fact that we do not need to construct a FO canonical Kripke model with value assignments anymore. 

\begin{definition}
The canonical model of $\SMLKvr$ is a tuple $$\M=\lr{S,\{\to_i:i \in\I\},\{R_i^c :i\in\I,c\in\C\},V}$$ \noindent where:
\begin{itemize}
\item $S$ is the set of all maximal \SMLKvr-consistent sets of \LKvr formulas,
\item $s\to_i t \iff \{\phi:\Box_i\phi\in s\}\subseteq t$,
\item $sR_i^c tu \iff$ (1)$\{\phi:\Box_i\phi\in s\}\subseteq t\cap u$ and (2)$\{\psi:\bBox{i}{c}\psi\in s\}\subseteq t\cup u$,
\item $V(s)=\{p:p\in s\}$.
\end{itemize}
\end{definition}
Note that condition (2) for $R^c_i$ says that if $s$ can see two $i$-accessible worlds which do not agree on $c$ then at least one should satisfy $\psi$ for each $\Kv_i(\neg\psi, c) \in s$. 
\begin{proposition}
The canonical model $\M$ is an $\MLKvr$ model.
\end{proposition}
\begin{proof}
We only need to check the three conditions of $R_i^c$. 
\begin{enumerate}
\item $sR_i^c uv\Rightarrow sR_i^c vu$: Obvious.
\item $sR_i^c uv \Rightarrow s\to_i u$: By condition (1) in the definition of $R_i^c$.
\item $sR_i^c uv$ and $s\to_i t \Rightarrow$ either $sR_i^c ut$ or $sR_i^c tv$: Suppose not. Then according to the definition of $R_{i}^{c}$, we have $\{\psi:\bBox{i}{c}\psi\in s\}\not\subseteq u\cup t$ and $\{\psi:\bBox{i}{c}\psi\in s\}\not\subseteq v\cup t$. So there exist $\psi_1,\psi_2\in \{\psi:\bBox{i}{c}\psi\in s\}$ such that $\psi_1\not\in u\cup t$ and $\psi_2\not\in v\cup t$. According to the property of maximal consistent sets, this entails $\psi_1\land\psi_2\not\in u\cup t$ and $\psi_1\land\psi_2\not\in v\cup t$. Now, we distinguish two situations: $\Diamond_i(\neg\psi_1\land\neg\psi_2)\in s$ and $\Diamond_i(\neg\psi_1\land\neg\psi_2)\not\in s$, and go on to show that in both cases we would arrive at contradiction.

Suppose $\Diamond_i(\neg\psi_1\land\neg\psi_2)\in s$. Note that since $\psi_1,\psi_2\in \{\psi:\bBox{i}{c}\psi\in s\}$, $\bBox{i}{c}\psi_1\in s$ and $\bBox{i}{c}\psi_2\in s$. Then according to $\KvrDIS$, we have $\bBox{i}{c}(\psi_1\land\psi_2)\in s$. So $\psi_1\land\psi_2\in\{\psi:\bBox{i}{c}\psi\in s\}$. Since $\{\psi:\bBox{i}{c}\psi\in s\}\subseteq u\cup v$, $\psi_1\land\psi_2\in u\cup v$. But this means that $\psi_1\land\psi_2\in u\cup t$ or $v\cup t$, contradiction.

Suppose $\Diamond_i(\neg\psi_1\land\neg\psi_2)\not\in s$, then $\Box_{i}(\psi_1\vee\psi_2)\in s$. According to the definition of $R_{i}^{c}$, we have $\psi_1\vee\psi_2\in t$. By the property of MCS, at least one of  $\psi_1$ and $\psi_2$ is in $t$. However, since $\psi_1\not\in u\cup t$ and $\psi_2\not\in v\cup t$, we have $\psi_1,\psi_2\not\in t$, contradiction.
\end{enumerate}
Therefore, the canonical model $\M$ is indeed an $\MLKvr$ model.
\end{proof}

By a Lindenbaum-like argument, every consistent set of $\MLKvr$ formulas can be extended to a maximal consistent set (of \MLKvr\ formulas). In the following we (as routine) prove the existence lemma for both modalities $\Dia_i$ and $\bDia{i}{c}$ in order to obtain the truth lemma. The proof is the $\SMLKvr$ adaption of the proof of $\SELKvr$ in \cite{WangF14}. 

Given a state $s\in S$ such that $\bDia{i}{c}\phi\in s$. We let $Z=\{\psi\mid \Box_i\psi\in s\}\cup\{\phi\}$ and $X=\{\chi \mid \bBox{i}{c}\chi\in s\}$. Since $X$ is countable, we list the elements in $X$ as $\chi_i$ for $i\in\natN$.
Note that since $\vdash\bBox{i}{c}\top$, $\top\in X$, namely $X$ is non-empty.

\begin{fact}
For any $\chi\in X$, $\{\chi\}\cup Z$ is consistent. Therefore $Z$ and every $\chi$ are also consistent.
\end{fact}

\begin{proof}
Suppose not, then there exists $\chi\in X$, $\psi_1,\dots,\psi_n\in Z$ such that $\vdash\psi_1\land\dots\land\psi_n\land\phi\to\neg\chi$. By $\GENK$ and $\DISTK$, we have $\vdash\Box_i(\psi_1\land\dots\land\psi_n)\to\Box_i(\phi\to\neg\chi)$. Since $\Box_i\psi_1,\dots,\Box_i\psi_n\in s$, $\Box_i(\phi\to\neg\chi)\in s$. Note that $\KvrDIS$ is equivalent to $\Box_i(p\to q)\to(\bDia{i}{c}p\to\bDia{i}{c}q)$. By \SUB, we have $\vdash\Box_i(\phi\to\neg\chi)\land\bDia{i}{c}\phi\to\bDia{i}{c}\neg\chi$. This together with the fact that $\Box_i(\phi\to\neg\chi)\in s$ and $\bDia{i}{c}\phi\in s$ (assumption), we have $\bDia{i}{c}\neg\chi\in s$, contradiction. Since $\top\in X$, $\{\top\}\cup Z$ is consistent thus $Z$ is consistent. 
\end{proof}

% \begin{fact}
% Every $\chi$ in $X$ is consistent.
% \end{fact}

% \begin{proof}
% Suppose not. Then $\exists\chi\in X$ such that $\vdash\chi\to\bot$. By $\GENK$ and $\DISTKv$, $\bBox{i}{c}\bot\in s$. But $\vdash\bDia{i}{c}\phi\to\bDia{i}{c}\top$ and $\bDia{i}{c}\phi\in s$ imply $bDia{i}{c}\top=\neg\bBox{i}{c}\bot\in s$, contradiction.
% \end{proof}

Let $B_0=Z\cup\{\chi_0\}$, $C_0=Z$. We inductively construct $B_n$ and $C_n$ as following:
\begin{itemize}
\item If $B_n\cup\{\chi_{n+1}\}$ is consistent, then $B_{n+1}=B_n\cup\{\chi_{n+1}\}$, $C_{n+1}=C_n$.

\item Else, $B_{n+1}=B_n$, $C_{n+1}=C_n\cup\{\chi_{n+1}\}$.

\item Finally, let $B=\bigcup_{n<\omega}B_n$, $C=\bigcup_{n<\omega}C_n$.
\end{itemize}
In order to show that $B$ and $C$ are consistent we first show that $B_n$ and  $C_n$ are consistent for each $n<\omega$. Before that we prove a handy proposition which will be useful soon.

\begin{proposition}\label{genKvr}
For any finitely many \MLKvr\ formulas $\psi_1,\dots,\psi_n$, we have: 
$$\vdash\bBox{i}{c}\psi_1\land\cdots\land\bBox{i}{c}\psi_n \land \Dia_i(\neg\psi_1\land\cdots\land\neg\psi_n)\to\bBox{i}{c}(\psi_1\land\cdots\land\psi_n)$$
\end{proposition}

\begin{proof}
From $\KvrDIS$ we have $\vdash\bBox{i}{c}\phi\land\bBox{i}{c}\psi\land\Dia_i(\neg\phi\land\neg\psi)\to\bBox{i}{c}(\phi\land\psi)$. What we want now is its generalization. By applying $\KvrDIS$, we can have the following sequence of $\SMLKvr$ theorems:
\begin{align*}
\vdash & \bBox{i}{c}\psi_1\land\bBox{i}{c}\psi_2\land\Dia_i(\neg\psi_1\land\neg\psi_2)\to\bBox{i}{c}(\psi_1\land\psi_2) \\
\vdash & \bBox{i}{c}(\psi_1\land\psi_2)\land \bBox{i}{c}\psi_3 \land \Dia_i(\neg(\psi_1\land \psi_2)\land\neg\psi_3) \to \bBox{i}{c}(\psi_1\land\psi_2\land\psi_3) \\
& \cdots \\
\vdash & \bBox{i}{c}(\psi_1\land\cdots\land\psi_{n-1})\land\bBox{i}{c}\psi_n \land \Dia_i(\neg(\psi_1\land\cdots\land\psi_{n-1})\land\neg\psi_n)\to\bBox{i}{c}(\psi_1\land\cdots\land\psi_n)
\end{align*}
Note that we have $\vdash\Dia_i(\neg\psi_1\land\cdots\land\neg\psi_n) \to \Dia_i(\neg(\psi_1\land\psi_{k-1})\land \neg \psi_k)$ in standard normal modal logic, for any $2\leq k\leq n$. Now by using the above theorems one by one we can obtain:
$$\vdash\bBox{i}{c}\psi_1\land\cdots\land\bBox{i}{c}\psi_n \land \Dia_i(\neg\psi_1\land\cdots\land\neg\psi_n)\to\bBox{i}{c}(\psi_1\land\cdots\land\psi_n)$$
\end{proof}

\begin{proposition}\label{prop.consis}
For any $k\geq 0$, if $B_k$ is consistent and $\chi_{k+1}$ is not consistent with $B_k$, then $\chi_{k+1}$ is consistent with $C_k$. Therefore $B_k$ and $C_k$ are consistent for $k\in \mathbb{N}$.
%\noteYW{De we need to assume that $B_k$ is consistent here?}
\end{proposition}

\begin{proof}
Suppose not, i.e., $\chi_{k+1}$ is not consistent with both $B_k$ and $C_k$. Let $U=B_k\setminus Z$, $V=C_k\setminus Z$, $\UU=\{\neg\psi:\psi\in U\}$, and $\VV=\{\neg\psi:\psi\in V\}$. Then there exist  $\alpha_1,\dots,\alpha_\ell,\beta_1,\dots,\beta_m,\gamma_1,\dots,\gamma_n\in Z$ such that:
\begin{itemize}
\item $\vdash\alpha_1\land\dots\land\alpha_\ell\land\bigwedge U\land\phi\to\neg\chi_{k+1}$
\item $\vdash\beta_1\land\dots\land\beta_m\land\bigwedge V\land\phi\to\neg\chi_{k+1}$
\item $\vdash\gamma_1\land\dots\land\gamma_n\land\bigwedge U\land\phi\to \bigwedge\VV$
\end{itemize}
The last one is due to the fact that any formula in $C_k\setminus Z$ is inconsistent with $B_k$ by construction. 
By $\GENK$, $\DISTK$ and the definition of $Z$ and $X$, we have
\begin{itemize}
\item $\Box_i(\bigwedge U\land\phi\to\neg\chi_{k+1})\in s$
\item $\Box_i(\bigwedge V\land\phi\to\neg\chi_{k+1})\in s$
\item $\Box_i(\bigwedge U\land\phi\to \bigwedge\VV)\in s$
\end{itemize}

First, we claim that $\Dia_i(\bigwedge U\land\phi)\in s$. If not, then $\Box_i\neg(\bigwedge U\land\phi)\in s$, which means that $\neg(\bigwedge U\land\phi)\in Z\subseteq B_k$. But as $U\subseteq B_k$, $\phi\in B_k$, this implies that $B_k$ is inconsistent, contradiction. 

Then, we claim that $\Dia_i(\neg\chi_{k+1}\land\bigwedge\VV)\in s$. Since $\Box_i(\bigwedge U\land\phi\to\neg\chi_{k+1})\in s$ and $\Box_i(\bigwedge U\land\phi\to \bigwedge\VV)\in s$ then $\Box_i(\bigwedge U\land\phi\to\neg\chi_{k+1}\land\bigwedge\VV)$, we immediately get $\Dia_i(\neg\chi_{k+1}\land\bigwedge\VV)\in s$ due to the fact $\Dia_i(\bigwedge U\land\phi)\in s$ that we just showed.

Finally, since $\Dia_i(\neg\chi_{k+1}\land\bigwedge\VV)\in s$, $\bBox{i}{c}\chi_{k+1}\in s$ and $\bBox{i}{c}\psi\in s$ for all $\psi\in V$, by Proposition~\ref{genKvr}, we have $\bBox{i}{c}(\chi_{k+1}\land\bigwedge V)\in s$. But since $\Box_i(\bigwedge V\land\phi\to\neg\chi_{k+1})\in s$, its equivalent $\Box_i(\bigwedge V\land\chi_{k+1}\to\neg\phi)$ is also in $s$. By $\DISTKvr$ we know $\vdash\Box_i(\bigwedge V\land\chi_{k+1}\to\neg\phi)\to \bBox{i}{c}(\bigwedge V\land\chi_{k+1})\to\bBox{i}{c}\neg\phi$. Therefore $\bBox{i}{c}\neg\phi\in s$, contradiction to $\bDia{i}{c}\phi\in s$. 

Now, we can prove that $B_k$ and $C_k$ are consistent for any $k\in\natN$. We do induction on $k$. For $k=0$, then $B_0=C_0=Z$, whose consistency is shown in \textit{Fact 3.3}. For $k=i+1$, consider whether $\chi_{i+1}$ is consistent with $B_i$. If $\chi_{i+1}$ is consistent with $B_i$, then $B_k=B_i\cup\chi_{i+1}$ and $C_k=C_i$ (by induction hypothesis) are consistent. If $\chi_{i+1}$ is inconsistent with $B_i$, then by induction hypothesis, $B_i=B_{i+1}$ and $C_i$ are consistent. So according to the above conclusion $C_{i+1}$ is also consistent.
\end{proof}

\begin{proposition}\label{prop.con}
$B=\bigcup_{n<\omega}B_n$ and $C=\bigcup_{n<\omega}C_n$ are both consistent.
\end{proposition}

\begin{proof}
Suppose $B$ is not consistent. That is, there exist $\phi_1,\dots,\phi_n\in B$ such that $\vdash\phi_1\land\dots\land\phi_n\to\bot$. Therefore, there must be a finite $m$ such that $\phi_1,\dots,\phi_n\in B_m$. But this means that $B_m$ is already inconsistent, contradictory to the construction of $B_k$. The case for $C$ is similar.
\end{proof}

It is routine to prove the following: 
\begin{lemma}(Existence Lemma for $\Dia_i$)
Given a state $s\in S$. 
If $\Dia_i\phi\in s$, then there exists $t\in S$ such that $s\to_i t$ and $\phi\in t$;
\end{lemma}

% \begin{proof}
% This proof is quite routine. First, we can easily show that $X=\{\phi\}\cup\{\chi:\Box_i\chi\in s\}$ is consistent. Then by the Lindenbaum Lemma for \LKvr, there exists a MCS $\Gamma$ such that $X\subseteq\Gamma$. Now, with the construction method in hand, we can construct a $t$ such that $\Gamma\in t$ and $s\to_i t$.
% \end{proof}
Also we have the existence lemma for  $\bDia{i}{c}$: 
\begin{lemma}(Existence Lemma for $\bDia{i}{c}$)
Given a state $s\in S$.
If $\bDia{i}{c}\psi\in s$, then there exist $t,u\in S$ such that $sR^c_i tu$ and $\psi\in t \cap u$.
\end{lemma}
\begin{proof}
Let $Z$, $B$ and $C$ be defined as above. Due to Proposition~\ref{prop.con} $B$ and $C$ are both  consistent. Therefore, both can be extended into maximal consistent sets, say $t$ and $u$. Now, the construction of $B$ and $C$ itself guarantee that $sR_i^c tu$ and $\phi\in t, u$. 
\end{proof}

\begin{lemma}(Truth Lemma)
For any state $s\in \M$ and $\phi$, $\M,s\Vdash\phi \iff \phi\in s$.
\end{lemma}

\begin{proof}
Prove by induction. We only give the $\bDia{i}{c}\psi$ case; the others are routine.

$\Rightarrow$: Suppose $\M,s\Vdash\bDia{i}{c}\psi$. Then there exist $t,u$ such that $sR_i^c tu$, $\M,t\Vdash\psi$ and $\M,u\Vdash\psi$. By induction hypothesis, $\psi\in t\cap u$. If $\bDia{i}{c}\psi\not\in s$, then $\bBox{i}{c}\neg\psi\in s$, which implies $\neg\psi\in t\cup u$ by the construction of $R_i^c$, contradiction. Therefore, $\bDia{i}{c}\psi\in s$.

$\Leftarrow$:Suppose $\bDia{i}{c}\psi\in s$. Then according to the existence lemma for $\bDia{i}{c}$, there exist $t,u$ such that $sR_i^c tu$ and $\psi\in t\cap u$. By induction hypothesis, $\M,t\Vdash\psi$, $\M,u\Vdash\psi$. Therefore $\M,s\Vdash\bDia{i}{c}\psi$.
\end{proof}

The completeness result then follows immediately:
\begin{theorem}(Completeness)
 $\SMLKvr$ is strongly complete over arbitrary models.
\end{theorem}
\begin{remark}
At this point, it is interesting to compare our canonical model with the canonical model used in \cite{WangF14}. A complication in \cite{WangF14} is that merely maximal consistent sets are not enough to build a FO canonical Kripke model. However, as we have seen, we only use the  maximal consistent sets in our canonical \MLKvr model: it does not involve value assignments. Thus we have restored the symmetry between the logical language and the model to some extent: there is no longer too much information in the model, which cannot be talked about by the language. Note that we allow $sR_itt$, which also helps to have compact models. 
\end{remark}

\section{Extended language with  binary modalities}
\label{Sec.gen}
In the previous sections, we treat $\bDia{i}{c}$ as a unary modality interpreted by a ternary relation.  Essentially, $\bDia{i}{c}$ can be viewed as a binary modality where the two arguments are the same. In this section, we restore the symmetry between the semantics and the syntax one step further by having the \textit{binary} $\bDia{i}{c}(\cdot, \cdot)$ in the language. Surprisingly, this  extension does \textit{not} increase the expressive power of $\MLKvr$. What is more, the new logic is normal. Consequently, the extension will help us to understand $\MLKvr$ more deeply from a normal modal logic point of view. 

The extended language $\MLKvb$ is given by the following BNF ($^b$ for \textit{binary}):

$$\phi::=\top\mid p\mid \neg\phi\mid (\phi\wedge\phi)\mid \Box_i\phi\mid \bBox{i}{c}(\phi,\phi)$$
We define $\bDia{i}{c}(\psi, \phi)$ as $\neg \bBox{i}{c}(\neg \psi, \neg\phi)$. And $\bbDia{i}{c}\phi$ is now equivalent to the  $\MLKvb$ formula $\bbDia{i}{c}(\phi, \phi)$. To see the intuition, for example, $\bDia{i}{c}(p, \neg p)$ says that $i$ can see a $p$ world and a $\neg p$ world which  do not agree on the value of $c$. Formally, the semantics is defined on the same \MLKvr\ models $\M=\lr{S,\{\to_i:i\in\I\},\{R_i^c:i\in\I,c\in\C\},V, V_\C}$:

{\begin{small} 
$$\begin{array}{|l@{\quad\Lra\quad}l|}
\hline
\M,s\Vdash\bbDia{i}{c}(\phi,\psi) & \text{ there exist $t,u\in S$ such that $sR_i^c tu$, $\M,t\Vdash\phi$ and $\M,u\Vdash\psi$.}\\
\hline
\end{array}
$$
\end{small}}
% First we give the model for \MLKvb. In fact, this is the same with \ELKvr. The difference of the two language lies in their expressive power.

% \begin{definition}
% A model $\M$ for \MLKvb is a tuple $\lr{S,\{\to_i:i\in\I\},\{R_i^c:i\in\I,d\in\D\},V}$, where
% \begin{itemize}
% \item $\lr{S,\{\to_i:i\in I\},V}$ is a standard Kripke model with one modality.
% \item For each $d\in\D$, $R_{i}^{d}$ is a triple relation over $S$ satisfying: (1)$sR_{i}^{d}uv$ only if $s\to_i u$ and $s\to_i v$ and (2)$sR_{i}^{d}uv$ and $s\to_i t$ implies that $sR_{i}^{d}ut$ or $sR_{i} ^{d}vt$.
% \end{itemize}
% \end{definition}

% Now we give the semantics for \MLKvb\ on the same models as before for $\MLKvr$. Intuitively, $\bbDia{i}{d}(\phi,\psi)$ says that $i$ can see two possible states, one with $\phi$ true and the other with $\psi$ true, such that the two states ate in different equivalence class w.r.t. the value of $d$.
%Therefore, we have the semantics for 

% Then we can understand the intuitive meaning of the dual, namely $\bbBox{i}{d}$. $\bbBox{i}{d}(\phi,\psi)$ holds if for all accessible states, when we pick out two states such that $\neg\phi$ is true on one and $\neg\psi$ is true on the other, then they must share the same value of $d$.

The above semantics coincides with the standard semantics for binary diamond modalities \cite{mlbook}.\footnote{Binary modalities appear in many modal logics, such as the until operator in temporal logic, and the relevant implication in relevance logic interpreted on Kripke models with a ternary relation.} 
Note that  $\bDia{i}{c}(\phi,\psi)$ is essentially different from $\bDia{i}{c}(\phi\vee\psi)$: the latter only says that there are two $\phi\lor\psi$-successors that have different values of $c$, but not necessarily one $\phi$ world and one $\psi$ world. So, on first sight, \LKvr seems to be weaker than $\MLKvb$.

However, we will show by the following lemma that $\MLKvr$ and $\MLKvb$ are equally expressive, by reducing the binary $\bbDia{i}{c}$ to the unary $\bbDia{i}{c}$ in presence of the diamond $\Dia_i$.  

\begin{lemma}\label{lem.red}
$\bDia{i}{c}(\phi,\psi)$ is equivalent to the disjunction of the following three formulas: 
\begin{enumerate}
\item $\bDia{i}{c}\phi\land\Dia_i\psi$
\item $\bDia{i}{c}\psi\land\Dia_i\phi$
\item $\Dia_i\phi\land\Dia_i\psi\land\neg\bDia{i}{c}\phi\land\neg\bDia{i}{c}\psi\land\bDia{i}{c}(\phi\vee\psi)$
\end{enumerate}
\end{lemma}

\begin{proof}
The proof consists of two directions. 

First, we show that each of the three disjuncts entails $\bbDia{i}{c}(\phi,\psi)$.
\begin{enumerate}
\item For any model $\M,s$ that satisfies $\bDia{i}{c}\phi\land\Dia_i\psi$, there exists $t,u\in S$ such that $sR_i^ctu$, $t\Vdash\phi$ and $u\Vdash\phi$, and exists $v$ such that $s\to_i v$ and $v\Vdash\psi$. According to the property of $R_i^c$, at least one of $sR_i^c tv$ and $sR_i^c uv$ holds. W.l.o.g. suppose $sR_i^c tv$. Then according to the semantics of $\bDia{i}{c}$, we have $s\Vdash\bDia{i}{c}(\phi,\psi)$.

\item For $\bDia{i}{c}\psi\land\Dia_i\phi$, the proof is similar to (\romannumeral1).

\item If $\M,s\Vdash \Dia_i\phi\land\Dia_i\psi\land\neg\bDia{i}{c}\phi\land\neg\bDia{i}{c}\psi\land\bDia{i}{c}(\phi\vee\psi)$, then: $s$ has $\phi$-successors and $\psi$-successors; all $\phi$-successors have the same value of $c$, all $\phi$-successors have the same value of $c$, but the two values are different due to $\M,s\Vdash \bDia{i}{c}(\phi\vee\psi)$. So we can easily guarantee that there are two states, one $\phi$-successor and one $\psi$-successor of $s$ such that they have different values with regard to $c$. This means $\M,s\Vdash\bbDia{i}{c}(\phi,\psi)$.
\end{enumerate}

Second, we prove that if $\M,s\Vdash\bDia{i}{c}(\phi,\psi)$, then at least one of (\romannumeral1), (\romannumeral2) and (\romannumeral3) holds.

Suppose $\M,s\Vdash\bDia{i}{c}(\phi,\psi)$, namely there exist $t,u\in\M$ such that $sR_i^c tu$, $\M,t\Vdash\phi$ and $\M,u\Vdash\psi$. We immediately have $\M,s\Vdash\Dia_i\phi\land\Dia_i\psi\land\bDia{i}{c}(\phi\lor\psi)$. If neither $\bDia{i}{c}\phi$ nor $\bDia{i}{c}\psi$ holds on $s$, then $\M,s\Vdash\Dia_i\phi\land\Dia_i\psi\land\neg\bDia{i}{c}\phi\land\neg\bDia{i}{c}\psi\land\bDia{i}{c}(\phi\vee\psi)$. Therefore, $\M,s\Vdash(\bDia{i}{c}\phi\land\Dia_i\psi)\lor (\bDia{i}{c}\psi\land\Dia_i\phi)\lor (\Dia_i\phi\land\Dia_i\psi\land\neg\bDia{i}{c}\phi\land\neg\bDia{i}{c}\psi\land\bDia{i}{c}(\phi\vee\psi))$

In sum, we can now conclude the equivalence.
\end{proof}

With this lemma in hand, the reduction theorem is straightforward:
\begin{theorem}(Reduction)
For any \MLKvb\ formula $\phi$, there exists an  \MLKvr\ formula $\psi$ such that for any pointed model $\M,s$: $\M,s\Vdash\phi \iff \M,s\Vdash\psi$.
\end{theorem}

\begin{proof}
We define a reduction function $r$ inductively:
\begin{itemize}
\item $r(p)=p$; $r(\neg\phi)=\neg r(\phi)$; $r(\phi\land\psi)=r(\phi)\land r(\psi)$; $r(\Dia_i\phi)=\Dia_i r(\phi)$;
\item $r(\bDia{i}{c}(\phi,\psi))=(\bDia{i}{c}r(\phi)\land\Dia_i r(\psi)) \vee (\bDia{i}{c}r(\psi)\land\Dia_i r(\phi)) \vee (\Dia_i r(\phi)\land\Dia_ir(\psi)\land\neg\bDia{i}{c}r(\phi)\land\neg\bDia{i}{c}r(\psi)\land\bDia{i}{c}(r(\phi)\vee r(\psi)))$.
\end{itemize}

The correctness of the reduction is guaranteed based on Lemma~\ref{lem.red}. It is not hard (but important) to see that the rewriting always terminates.
\end{proof}

\begin{remark}
Although $\MLKvr$ is equally expressive as \MLKvb\ in presence of $\Dia_i$, it is not the case if $\Dia_i$ is absent. To see this, consider the following two pointed models $\M,s$ and $\N,x$ where $sR_i^ctu, sR_i^cuv$ in the left model, and $xR_i^cyz$ in the other model:
$$
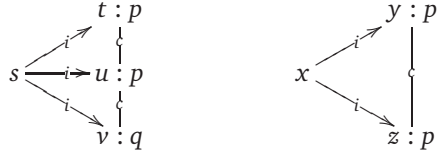
\xymatrix@R-11pt{
& t:p\ar@{-}[d]|c && & y:p\ar@{-}[dd]|c\\
s\ar[ur]|i \ar[r]|i \ar[dr]|i & u:p && x\ar[ur]|i \ar[dr]|i\\
& v:q\ar@{-}[u]|c && & z:p
}$$

We can use $\bDia{i}{c}(p,q)$ to distinguish the two pointed models; however, they are indistinguishable by using any formula with the unary $\bDia{i}{c}$ but no $\Dia_i$, which can be proved by a simple induction. 
%\noteTao{This still require proof. To do that, a formal way is to consider the new language without normal $\Box$ operator, and try to give a definition of bisimulation there.}
% Note that $\Dia_i$'s are important in the reduction formula. One reason behind is that both binary $\bDia{i}{c}$ and unary $\bDia{i}{c}$ interact with the ordinary $\Dia_i$ operator. Therefore, even if unary $\bDia{i}{c}$ itself is strictly less expressive than binary $\bbDia{i}{c}$, we can still expect the two languages added with $\Dia_i$ to be equally expressive. These insights can help us later in formulating systems without $\Dia_i$ operator.
\end{remark}

Now observe that $\MLKvb$ is a standard modal language defined on standard Kripke models with standard semantics. It is a relatively routine exercise to propose a normal modal logic system with the following axioms $\EQUIV$, $\EXIST$ and $\DISBK$ to capture the corresponding special properties of the models:\\ 
\begin{minipage}{0.64\textwidth}
\begin{center}
\begin{tabular}{lc}
\multicolumn{2}{c}{System $\SMLKvb$}\\
% \lcline{1-2}\rcline{3-4}
\multicolumn{2}{l}{Axiom Schemas}\\
%\lcline{1-2}\rcline{3-4}
\texttt{TAUT} & \tr{ all the instances of tautologies}\\
\DISTK &$ \Box_i (p\to q)\to (\Box_i p\to \Box_i q)$\\
\DISTKvb &$ \bBox{i}{c}(p\to q,r)\to (\bBox{i}{c}(p,r)\to \bBox{i}{c}(q,r))$\\
$\EQUIV$ &$\bbBox{i}{c}(p,q)\to\bbBox{i}{c}(q,p)$\\
% $\KvrTr$ &$\bBox{i}{c}\phi \to \Box_i\bBox{i}{c}\phi$\\
$\EXIST$ &$\bDia{i}{c}(p,q)\to\Dia_i p$\\
$\DISBK$ &$\bDia{i}{c}(p,q)\land\Diamond_i r \to  \bDia{i}{c}(p,r)\lor \bDia{i}{c}(q,r)$\\
\end{tabular}
\end{center}
\end{minipage}
\hfill
\begin{minipage}{0.33\textwidth}
\begin{center}
\begin{tabular}{lc}
\multicolumn{2}{l}{Rules}\\
\texttt{MP}& $\dfrac{\phi,\phi\to\psi}{\psi}$\\
\GENK &$\dfrac{\phi}{\Box_i \phi}$\\
\NECKvb &$\dfrac{\phi}{\bBox{i}{c}(\phi,\psi)}$\\
\SUB &$\dfrac{\phi}{\phi [p\slash\psi]}$\\
\RE &$\dfrac{\psi\lra\chi}{\phi\lra \phi[\psi\slash\chi]}$\\
\end{tabular}
\end{center}
\end{minipage}\\

Note that due to $\EQUIV$, we do not need to include the variations of \DISTKvb\ and \NECKvb\ w.r.t.\ the second argument in the binary $\bBox{i}{c}$ (cf. \cite{mlbook} for the standard proof systems of polyadic normal  modal logics.)

In this system $\SMLKvb$ we can derive all the axioms in $\SMLKvr$. Before proving it, we first show the following handy propositions. 
\begin{proposition}\label{prop.dis}
$\vdash_{\SMLKvb} \bDia{i}{c}(p\lor q,r)\to\bDia{i}{c}(p,r)\lor\bDia{i}{c}(q,r)$.
\end{proposition}

\begin{proof}
This proposition captures the interaction between boolean operator $\lor$ and $\bDia{i}{c}$ \footnote{Actually this is a standard axiom for normal modal logic. In case the binary case might not be that familiar, we give the proof here.}. So we can only start from the axiom \DISTKvb. Note that $\RE$ is used frequently. 

(1) $\bBox{i}{c}(p\to q,r)\to (\bBox{i}{c}(p,r)\to \bBox{i}{c}(q,r))$ (\DISTKvb)

(2) $\neg(\bBox{i}{c}(p,r)\to \bBox{i}{c}(q,r))\to\neg\bBox{i}{c}(p\to q,r)$ 

(3) $\bBox{i}{c}(p,r)\land\bDia{i}{c}(\neg q,\neg r)\to\bDia{i}{c}(p\land \neg q,\neg r)$

(4) $\bDia{i}{c}(\neg q,\neg r)\to(\neg\bBox{i}{c}(p,r)\lor\bDia{i}{c}(p\land \neg q,\neg r))$

(5) $\bDia{i}{c}(\neg q,\neg r)\to(\bDia{i}{c}(\neg p,\neg r)\lor\bDia{i}{c}(p\land \neg q,\neg r))$

(6) $\bDia{i}{c}(p\lor q,r)\to(\bDia{i}{c}(p,r)\lor\bDia{i}{c}(\neg p\land(p\lor q),r))$ ((5) \& \SUB)

(7) $\bDia{i}{c}(p\lor q,r)\to(\bDia{i}{c}(p,r)\lor\bDia{i}{c}(q,r))$
\end{proof}

\begin{proposition}\label{prop.con1}
$\vdash_{\SMLKvb} \bDia{i}{c}(p\land q,r)\to\bDia{i}{c}(p,r)\land\bDia{i}{c}(q,r)$.
\end{proposition}

\begin{proof}
This is similar to the above proof.
\end{proof}

\begin{proposition}\label{prop.46}
$\vdash_{\SMLKvb}\bBox{i}{c}(p,r)\land\bBox{i}{c}(q,r)\land\Dia_i \neg r \to \bBox{i}{c}(p,q)$.
\end{proposition}
\begin{proof}
Easily derived from $\DISBK$: $\bDia{i}{c}(p,q)\land\Diamond_i r \to  \bDia{i}{c}(p,r)\lor \bDia{i}{c}(q,r)$.
\end{proof}

\begin{proposition}
All the $\SMLKvr$ axioms are provable in $\SMLKvb$ and the rules of $\SMLKvr$ are admissible in $\SMLKvb$ (viewing  $\bDia{i}{c}\phi$ as $\bDia{i}{c}(\phi,\phi)$).
\end{proposition}

\begin{proof}
We need to check \DISTKvr, \KvrDIS\ and \NECKvr\ in $\SMLKvr$.

\begin{enumerate}
\item \DISTKvr: 
% (1) $\bBox{i}{c}(p\to q,p)\to(\bBox{i}{c}(p,p)\to\bBox{i}{c}(q,p))$ (\DISTKvr)

% (2) $\bBox{i}{c}(p\to q,q)\to(\bBox{i}{c}(p,q)\to\bBox{i}{c}(q,q))$ (\DISTKvr)

(1)  $\bDia{i}{c}(p,q)\to\Dia_i p$ ($\EXIST$)

(2) $\Box_i \neg p\to \bBox{i}{c}(\neg p, \neg q)$

(3) $\Box_i (p\to q)\to \bBox{i}{c}(p\to q,p)$ ((2) \SUB)

(4) $\Box_i (p\to q)\to \bBox{i}{c}(p\to q,q)$ ((2) \SUB)

(5) $\Box_i (p\to q)\to (\bBox{i}{c}(p\to q,p)\land \bBox{i}{c}(p\to q,q))$ ((3) (4))

(6) $\Box_i (p\to q)\to (\bBox{i}{c}(p,p)\to\bBox{i}{c}(q,p))\land (\bBox{i}{c}(p,q)\to\bBox{i}{c}(q,q))$ (\DISTKvb)

(7) $\Box_i (p\to q)\to (\bBox{i}{c}p \to \bBox{i}{c}q)$ ((6) $\EQUIV$)

\item \KvrDIS: 

(1) $\bDia{i}{c}(p,q)\land\Dia_ir\to \bDia{i}{c}(q, r)\lor \bDia{i}{c}(p,r)$ ($\DISBK$)

(2) $\bDia{i}{c}(p\lor q, p\lor q)\land\Dia_i(p\land q)\to \bDia{i}{c}(p\lor q, p\land q)$ (\SUB)

(3) $\bDia{i}{c}(p\lor q)\land\Dia_i(p\land q)\to (\bDia{i}{c}(p, p\land q)\lor \bDia{i}{c}(q, p\land q))$ (Prop.~\ref{prop.dis})

(4) $\bDia{i}{c}(p\lor q)\land\Dia_i(p\land q)\to (\bDia{i}{c}(p,p)\lor \bDia{i}{c}(q,q))$ (Prop.~\ref{prop.con1})

(5) $\bDia{i}{c}(p\lor q)\land\Dia_i(p\land q)\to (\bDia{i}{c}p\lor \bDia{i}{c}q)$
% (1) $\bDia{i}{c}(p\lor q)\land\Dia_i(p\land q)\to \bDia{i}{c}(p\lor q,p\land q)$ (\KvrDIS \& \SUB)

% (2) $\bDia{i}{c}(p\lor q,p\land q)\to \bDia{i}{c}(p,p\land q)\lor\bDia{i}{c}(q,p\land q)$

% (3) $\bDia{i}{c}(p\lor q,p\land q)\to \bDia{i}{c}p\lor\bDia{i}{c}q$

\item \NECKvr: It is a special case of \NECKvb\ in \SMLKvb\ where the two arguments are the same.
\end{enumerate}
Other axioms and rules in $\SMLKvr$ are exactly the same as in $\SMLKvb$.
\end{proof}
Now as we can see below, the standard technique suffices to prove the completeness of $\SMLKvb$. The only tricky point is the ternary canonical relation.
\begin{theorem}
$\SMLKvb$ is sound and strongly complete w.r.t.\ $\MLKvr$ models.
\end{theorem}

\begin{proof} The soundness is straightforward to check. For the completeness we build a canonical model: $$\M=\lr{S,\{\to_i:i \in\I\},\{R_i^c :i\in\I,c\in\C\},V_\C}$$ %\noindent where:
\begin{itemize}
\item $S$ is the set of all maximal $\SMLKvb$-consistent sets of $\MLKvb$ formulas,
\item $s\to_i t \iff \{\phi:\Box_i\phi\in s\}\subseteq t$,
\item $sR_i^c tu \iff$ (1) $\{\phi:\Box_i\phi\in s\}\subseteq t\cap u$ and (2) for any $\bBox{i}{c}(\phi, \psi)\in s$, $\phi\in t$ or $\psi\in u$.
\item $V_\C(s)=\{p:p\in s\}$.
\end{itemize}

Note that the existence lemma for $\bDia{i}{c}$ is quite routine for normal polyadic modal logic, cf. \cite{mlbook}. The idea is to build two $i$-successors $t,u$ of $s$ if $\bDia{i}{c}(\phi, \psi)\in s$, such that $\phi\in t$, $\psi\in u$ and $sR_i^ctu$. According to the method in \cite[pp.\ 200]{mlbook}, we can build two maximal consistent sets $t, u$ such that $\phi\in t$ and $\psi\in u$, and for all $\bBox{i}{c}(\chi_1, \chi_2)\in s$ we have  $\chi_1\in t$ or $\chi_2\in u$. To make sure $sR_i^c tu$ we just need to check condition (1). To see this, note that by $\EXIST$ we have $\Box_i\chi\to \bBox{i}{c}(\chi, \theta)\in s$ and by $\EQUIV$ we have $\Box_i\chi\to \bBox{i}{c}(\theta, \chi)\in s$. Therefore for each $\Box_i\chi\in s$ we have $\bBox{i}{c}(\chi, \bot)\in s$ and $\bBox{i}{c}(\bot, \chi)\in s$. Due to the  construction of maximal consistent sets $t,u$, we have   $\chi\in t$ or $\bot\in u$, and $\bot\in t$ or $\chi\in u$, which implies $\chi\in t\cap u$. Thus $\{\chi:\Box_i\chi\in s\}\subseteq t\cap u$. This concludes the proof that $sR^c_itu.$ Based on the existence lemmas for both $\Dia_i$ and $\bDia{i}{c}$ we can prove the truth lemma $\phi\in s\iff s\Vdash\phi$ using standard techniques. 

In the rest of this proof we verify that the canonical model satisfies the three properties of $\MLKvr$ models. Note that condition (1) in the definition of $R^c_i$ is symmetric, and condition (2) is also implicitly symmetric due to axiom $\EQUIV$. It is also obvious that $sR_i^ctu$ implies $sR_it$ and $sR_iu$ by definition. We only need to verify the anti-euclidean property.
% We prove it indirectly, with the help of the completeness result of $\SMLKvr$. That is, $$\Sigma\vDash\phi \Rightarrow r(\Sigma)\vDash\phi \Rightarrow r(\Sigma)\vdash_{\SMLKvr}r(\phi) \Rightarrow \Sigma\vdash_{\SMLKvb}\phi$$
% The second implication is the completeness result we've already had. So we only need to prove the second one. To do this, the basic idea is to transform each proof in $\SMLKvr$ into one in $\SMLKvb$, since the two systems are essentially the same. And we've already proved that the translation of all axioms and rules in $\SMLKvr$ can be proved in $\SMLKvb$.

Towards contradiction suppose $sR_i^c tu$, $s\to_i v$ but neither $sR_i^c tv$ nor $sR_i^c uv$. Then according to the definition of $R_i^c$, there exist $\bBox{i}{c}(\phi_1,\psi_1),\bBox{i}{c}(\phi_2,\psi_2)\in s$ such that $\neg\phi_1\in t, \neg\psi_1\in v$, $\neg\phi_2\in u$  and $\neg\psi_2\in v$. Therefore $\neg\psi_1\land\neg\psi_2\in v$. Since $s\to_i v$,  $\Dia_i(\neg\psi_1\land\neg\psi_2)\in s$. By \DISTKvb, $\EQUIV$, and \NECKvb, it is not hard to show $\vdash_\SMLKvb \bBox{i}{c}(\phi_1,\psi_1)\to\bBox{i}{c}(\phi_1,\psi_1\lor\psi_2)$ and $\vdash_\SMLKvb \bBox{i}{c}(\phi_2,\psi_2)\to\bBox{i}{c}(\phi_2,\psi_1\lor\psi_2)$. So $\bBox{i}{c}(\phi_1,\psi_1\lor\psi_2), \bBox{i}{c}(\phi_2,\psi_1\lor\psi_2)\in s$. By Proposition~\ref{prop.46} and \SUB, $\vdash_\SMLKvb \bBox{i}{c}(\phi_1,\psi_1\lor\psi_2)\land\bBox{i}{c}(\phi_2,\psi_1\lor\psi_2)\land\Dia_i(\psi_1\lor\psi_2)\to \bBox{i}{c}(\phi_1,\phi_2)$. Since $\bBox{i}{c}(\phi_1,\psi_1\lor\psi_2)$, $\bBox{i}{c}(\phi_2,\psi_1\lor\psi_2)$ and $\Dia_i(\psi_1\lor\psi_2)$ are all in $s$, $\bBox{i}{c}(\phi_1,\phi_2)\in s$. This together with $sR_i^c tu$ imply that $\phi_1\in t ~\text{or}~ \phi_2\in u$, contradictory to the  assumption that $\neg\phi_1\in t$ and $\neg\phi_2\in u$.
\end{proof}

\section{Applications}
\subsection{Bisimulation}
In the field of modal logic, various bisimulation notions help to characterize the expressive power of the new semantics-driven logics. As a normal modal logic, $\MLKvb$ has a natural  notion of bisimulation (cf.\ \cite{mlbook}), and it will in turn help us to find a notion of bisimulation over FO Kripke models for the original \ELKvr.
\begin{definition}[$\C$-Bisimulation]
Let $\M_1=\lr{S_1,\{\to_{i}^{1}:i\in I\},\{R_{i}^{c}:i\in I, c\in\C\},V_1}$, $\M_2=\lr{S_2,\{\to_{i}^{2}:i\in I,c\in\C\},\{Q_{i}^{c}:i\in I\},V_2}$ be two models for $\MLKvb$ (also for $\MLKvr$). A $\C$-bisimulation between $\M_1$ and $\M_2$ is a non-empty binary relation $\Z\subseteq S_1\times S_2$ such that for all $s_1\Z s_2$, the following conditions are satisfied:
\begin{description}
\item[Inv]: $V_1(s_1)=V_2(s_2)$;
\item[Zig]: $s_1\to_{i}^{1} t_1 \Rightarrow$ $\exists t_2$ such that $s_2\to_{i}^{2} t_2$ and $t_1\Z t_2$;
\item[Zag]: $s_2\to_{i}^{2} t_2 \Rightarrow$ $\exists t_1$ such that $s_1\to_{i}^{1} t_1$ and $t_1\Z t_2$;
\item[Kvb-Zig]: $s_1 R_{i}^{c} t_1 u_1 \Rightarrow$ $\exists t_2, u_2\in S_2$ such that $t_1\Z t_2$, $u_1\Z u_2$ and $s_2 Q_{i}^{c} t_2u_2$;
\item[Kvb-Zag]: $s_2 Q_{i}^{c} t_2 u_2 \Rightarrow$ $\exists t_1, u_1\in S_1$ such that $t_1\Z t_2$, $u_1\Z u_2$ and $s_1 R_{i}^{c} t_1u_1$.
\end{description}
We say $\M,s$ and $\N,t$ are $\C$-bisimilar ($\M,s\bis_\C \N,t$) if there is a $\C$-bisimulation $Z$ between $\M$ and $\N$ and $(s,t)\in Z$. 
\end{definition}

% Now we need to show that the c-bisimuluation defined above is a real "bisimulation":

% \begin{theorem}
% If $\M_1,s_1\cbis\M_2,s_2$, then $\M_1,s_1\equiv_{\LKvr}\M_2,s_2$.
% \end{theorem}

% \begin{proof}
% Suppose $\M_1,s_1\cbis\M_2,s_2$. We prove by induction.
% \begin{enumerate}
% \item If $\phi=p$, then by \textbf{Invariance}, $\M_1,s_1\vDash\phi$.
% \item The boolean cases are routine.
% \item If $\phi=\Box_i\psi$, then by $\textbf{Zig}$\&$\textbf{Zag}$ this is done.
% \item If $\phi=\bDia{i}{c}\psi$, then consider two directions.\\
% Suppose $\M_1,s_1\vDash\bDia{i}{c}\psi$, then $\exists t_1, u_1\in S_1$ such that $s_1R_{i}^{c} t_1u_1$ with $\M_1,t_1\vDash\psi$ and $\M_1,u_1\vDash\psi$. By \textbf{Kv-Zig}, there exist $t_2, u_2\in S_2$ such that $t_1\Z t_2$, $u_1\Z u_2$ and $s_2 Q_{i}^{c} t_2u_2$. By induction hypothesis, $\M_2,t_2\vDash\psi$ and $\M_2,u_2\vDash\psi$. Since $s_2 Q_{i}^{c} t_2u_2$, $\M_2,s_2\vDash\bDia{i}{c}\psi$.\\
% The other direction from $\M_2,s_2\vDash\bDia{i}{c}\psi$ to $\M_1,s_1\vDash\bDia{i}{c}\psi$ by applying \textbf{Kv-Zag}.
% \end{enumerate}

% Therefore, $\M_1,s_1\vDash\phi$ iff $\M_2,s_2\vDash\phi$, for all $\phi\in\LKvr$, which means that $\M_1,s_1\equiv_{\LKvr}\M_2,s_2$.
% \end{proof}

% For \MLKvb, we have a similar result:

\begin{theorem}
If $\M_1,s_1\cbis\M_2,s_2$, then $\M_1,s_1\equiv_{\MLKvb}\M_2,s_2$.
\end{theorem}

\begin{proof}
Suppose $\M_1,s_1\cbis\M_2,s_2$. We prove by induction on the structure of \MLKvb\ formulas, and the only non-trivial case is when $\phi=\bbDia{i}{c}(\psi,\chi)$.

Suppose $\M_1,s_1\Vdash\bDia{i}{c}(\psi,\chi)$, then $\exists t_1,u_1\in S_1$ such that $s_1R_i^ct_1u_1$ with $\M_1,t_1\Vdash\psi$ and $\M_1,u_1\Vdash\chi$. By \textbf{Kvb-Zig}, there exist $t_2, u_2\in S_2$ such that $t_1\Z t_2$, $u_1\Z u_2$ and $s_2 Q_{i}^{c} t_2u_2$. By induction hypothesis, $\M_2,t_2\Vdash\psi$ and $\M_2,u_2\Vdash\chi$. Therefore, $\M_2,s_2\Vdash\bbDia{i}{c}(\psi,\chi)$. The other side is similar by \textbf{Kvb-Zag}.
\end{proof}

As in normal modal logic, we have the following theorem for \MLKvb\ (we omit the rather standard proof, but one can try to see how the binary modality $\bDia{i}{c}$ facilitates the proof):
\begin{theorem}
Suppose $\M$, $\N$ are finite models. Then $\M,s\cbis\N,t \iff \M,s\equiv_{\MLKvb}\N,t$. 
\end{theorem}

Since \LKvr\ and \MLKvb\ have the same expressive power we immediately have: 
\begin{corollary}
Suppose $\M$ and $\N$ are finite models. Then $\M,s\cbis\N,t \iff \M,s\equiv_{\LKvr}\N,t$.
\end{corollary}

% \begin{proof}
% Since \LKvr and \MLKvb have the same expressive power, they have the same model distinguishing power as well. That is, $\M,s\equiv_{\LKvr}\N,t \iff \M,s\equiv_{\MLKvb}\N,t$. Therefore, according to the above theorem, $\M,s\cbis\N,t \iff \M,s\equiv_{\MLKvb}\N,t$.
% \end{proof}

In \cite{WF13}, a notion of bisimulation has been offered for $\ELKv$, the epistemic logic with unconditional $\Kv_i$ operators. However, for $\ELKvr$ it was not that clear about the suitable bisimulation notion. Now we can recast $\C$-bisimulation back to the setting of $\ELKvr$ over FO Kripke models since $\ELKvr$ and $\MLKvr$ are essentially the same language. 
\begin{definition}[$\C$-bisimulation over FO Kripke models]
Given two pointed FO Kripke models $\M=\langle S_1,D_1,\{\to^1_i:i\in \I\}, V_1, V^1_\C\rangle$, and $\N=\langle S_2,D_2,\{\to^2_i:i\in \I\}, V_2, V^2_\C\rangle$, a relation $\Z\subseteq S_1\times S_2$ is a $\C$-bisimulation between the two models $\M, \N$ if whenever $s_1\Z s_2$ we have:
\begin{description}
\item[Inv] $V_1(s_1)=V_2(s_2)$;
\item[Zig]: $s_1\to_{i}^{1} t_1 \Rightarrow$ $\exists t_2$ such that $s_2\to_{i}^{2} t_2$ and $t_1\Z t_2$;
\item[Zag]: $s_2\to_{i}^{2} t_2 \Rightarrow$ $\exists t_1$ such that $s_1\to_{i}^{1} t_1$ and $t_1\Z t_2$;
\item[Kvr-Zig]If $s_1\to^1_it_1$ and $s_1\to^1_iu_1$ and $V^1_\C(c, t_1)\not=V^1_\C(c, u_1)$ then there are $t_2$ and $u_2$ in $\N$ such that $s_2\to^2_it_2$, $s_2\to^2_iu_2$, $t_1Zt_2$, $u_1Zu_2$, and $V^2_\C(c, t_2)\not=V^2_\C(c, u_2)$ in $\N$.
\item[Kvr-Zag] If $s_2\to^2_it_2$ and $s\to^2_iu_2$ and $V^2_\C(c, t_2)\not=V^2_\C(c, u_2)$ then there are $t_1$ and $u_1$ in $\M$ such that $s_1\to^1_it_1$, $s\to^1_iu_1$, $t_1Zt_2$, $u_1Zu_2$ and $V^1_\C(c, t_1)\not=V^1_\C(c, u_1)$ in $\M$.
\end{description}
Abusing the notation, FO Kripke models $\M,s$ and $\N,t$ are $\C$-bisimilar ($\M,s\bis_{\C}\N,t$) iff there exists a $\C$-bisimulation $Z$ between $\M$ and $\N$ such that $s, t\in Z$.
\end{definition}

Now since $\MLKvb$ and $\MLKvr$ have exactly the same expressive power, and $\MLKvr$ is equivalent to $\ELKvr$ modulo translation. The above $\C$-bisimilation works for $\ELKvr$, as proved in detailed in \cite{GvEW16}: 

\begin{theorem}
For finite FO Kripke models $\M_1, \M_2$:   $\M_1,s_1\bis_{\C}\M_2,s_2$ iff  $\M_1,s_1\equiv_{\ELKvr}\M_2,s_2$.
\end{theorem}

% but we have seen that it actually correspond to our extended language $\MLKvb$ more than $\LKvr$. So it is natural to investigate closely about their expressive power.

% One more word to say. Compared with the notion of bisimulation in Wang's paper, our definition of \textit{c-bisimulation} here is simpler and more direct. The reason is that when we ignore the real valuation of $c$'s and focus only on the equivalence relation, which is expressed in the version of difference, we have captured the essence of "knowing what" in our language $\LKvr$.

%\section{Characterization}
%\begin{definition}(Standard Translation)
%Let $x$ be a first-order variable. The standard translation $\st{x}$ taking \LKvr formulas to first-order formulas is defined as follows:

%$\st{x}(p)=Px$\\
%$\st{x}(\top)=x=x$\\
%$\st{x}(\phi\land\psi)=\st{x}(\phi)\land ST{x(\psi)}$\\
%$\st{x}(\Dia_i\phi)=\exists y(R_ixy\land\st{y}\phi)$\\
%$\st{x}(\bDia{i}{c}\phi)=\exists y\exists z(R_i^cxyz\land\st{y}\phi\land\st{z}\phi)$
%\end{definition}

\subsection{Completeness of \SMLKv}
The unconditional $\Kv$ operator was introduced in \cite{Plaza89:lopc} in the context of epistemic logic (call the language \ELKv):
$$\phi::=\top\mid p\mid \neg\phi\mid (\phi\wedge\phi)\mid \K_i\phi\mid \Kv_ic$$
Essentially, $\Kv_ic$ is $\Kv_i(\top, c)$ in $\ELKvr$. The semantics is as in the case of \ELKvr, which is based on FO Kripke models. Plaza gave two axioms on top of $\mathbb{S}5$ which are the counterparts of the introspection axioms in standard epistemic logic (over FO epistemic models): 
$$\Kv_ic\to K_i \Kv_ic \qquad \qquad \neg\Kv_ic\to K_i\neg \Kv_ic$$
% $$\phi::=\top\mid p\mid \neg\phi\mid (\phi\wedge\phi)\mid \Box_i\phi\mid \bBox{i}{c}\phi$$
However, neither \cite{Plaza89:lopc} nor \cite{WF13,WangF14} gave a complete proof of this simple logic. Here we look at this language from our $\bDia{i}{c}$ perspective, and consider the corresponding simple language ($\MLKv$) over the class of all the models: 
$$\phi::=\top\mid p\mid \neg\phi\mid (\phi\wedge\phi)\mid \Box_i\phi\mid \bBox{i}{c}\bot$$
Note that $\neg \Kv_i c$ can be viewed as $\bDia{i}{c}\top$. Thus $\Kv_i c$ is indeed $\bBox{i}{c}\bot$. 

% The model class $\MLKv$ is essentially $\MLKvr$ (this can be proved exactly like the transformations between $\MLKvr$ and $\ELKvr$ models). 
The semantics is just as in \MLKvr\ but we only allow $\top$ as the argument for $\bDia{i}{c}$. As in the case of $\ELKvr$ and $\MLKvr$ we can simply show that: 
\begin{proposition}\label{prop.elkv}
For any set of $\ELKv$ formula $\Sigma\cup\{\phi\}$, $\Sigma\vDash\phi$ iff $T(\Sigma)\Vdash T(\phi).$ 
\end{proposition}

Now based on this view, a natural system $\SMLKv$  is obtained by simplifying the system $\SMLKvr$:\\

\begin{minipage}{0.64\textwidth}
\begin{center}
\begin{tabular}{lc}
\multicolumn{2}{c}{System \SMLKv}\\
% \lcline{1-2}\rcline{3-4}
\multicolumn{2}{l}{Axiom Schemas}\\
%\lcline{1-2}\rcline{3-4}
\texttt{TAUT} & \tr{ all the instances of tautologies}\\
\DISTK &$ \Box_i (p\to q)\to (\Box_i p\to \Box_i q)$\\
\texttt{INCLT} & $\bDia{i}{c}\top\to\Dia_i\top$
\end{tabular}
\end{center}
\end{minipage}
\hfill
\begin{minipage}{0.30\textwidth}
\begin{center}
\begin{tabular}{lc}
\multicolumn{2}{l}{Rules}\\
\texttt{MP}& $\dfrac{\phi,\phi\to\psi}{\psi}$\\
\GENK &$\dfrac{\phi}{\Box_i \phi}$\\
\SUB &$\dfrac{\phi}{\phi [p\slash\psi]}$\\
\RE &$\dfrac{\psi\lra\chi}{\phi\lra \phi[\psi\slash\chi]}$\\
\end{tabular}
\end{center}
\end{minipage}
Note that due to the fact that the only $\bDia{i}{c}$ formula is $\bDia{i}{c}\top$ (and $\bBox{i}{c}\bot$), most of the previous axioms and rules do not apply. We only need to add one axiom \texttt{INCLT} on top of the usual normal modal logic, inspired by the \texttt{INCL} axioms of \SMLKvb. 

We go on to prove the completeness of $\SMLKv$.

\begin{definition}
The canonical model $\M$ is a tuple $\lr{S,\{\to_i:i\in\I\},\{R_i^c:i\in\I, c\in\C\},V}$ where:

\begin{itemize}
\item $S$ is the set of maximal $\SMLKv$-consistent sets,
\item $s\to_i t \iff \{\phi:\Box_i\phi\in s\}\subseteq t$,
\item $sR_i^c tt' \iff s\to_i t$, $s\to_i t'$ and $\bDia{i}{c}\top\in s$,
\item $V(s)=\{p:p\in s\}$
\end{itemize}
\end{definition}

The only tricky point is the definition of canonical relations $R_i^c$. The intuition is that as long as a state can see two states having different values, then we can safely assume all the states that it can see have different values. Note that in $\MLKvr$ models we also allow $sR_i^c tt$. We first need to verify that $\M$ is an $\MLKvr$ model:

\begin{proposition}
The canonical model $\M$ is an $\MLKvr$ model.
\end{proposition}

\begin{proof}
We only need to verify the three conditions. The first two are again obvious by definition, so we only prove the anti-euclidean property. Suppose $sR_i^c tt'$ and $s\to_i u$. Then $s\to_i t$, $s\to_i t'$ and $\bDia{i}{c}\top\in s$. So both $sR_i^c tu$ and $sR^c_i t'u$ by the definition of $R_i^c$.
\end{proof}

The existence lemma for $\Dia_i$ is routine.
% \begin{lemma}(Existence Lemma for $\Dia_i$)
% If $\Dia_i\phi\in s$, then there exists $t$ such that $s\to_i t$ and $\phi\in t$.
% \end{lemma}
As for the case of $\bDia{i}{c}$: 
\begin{lemma}(Existence Lemma for $\bDia{i}{c}$)
If $\bDia{i}{c}\top\in s$, then there exist $t,u$ such that $s R_i^c tu$.
\end{lemma}

\begin{proof}
Suppose $\bDia{i}{c}\top\in s$ then due to \texttt{INCLT}  $\Dia_i\top\in s$. By the existence lemma for $\Dia_i$, it follows that there exists $t$ such that $s\to_i t$. Therefore by definition $sR_i^c tt$. 
\end{proof}

\begin{lemma}
For any $s$ in $\M$ and $\MLKv$ formula $\phi$, $\M,s\Vdash\phi \iff \phi\in s$.
\end{lemma}
\begin{proof}
The only interesting case is when $\phi=\bDia{i}{c}\top$. 

$\Rightarrow$: Suppose $\M,s\Vdash\bDia{i}{c}\top$. Then there exist $t,t'$ such that $sR_i^c tt'$. By the definition of $R_i^c$, $\bDia{i}{c}\top\in s$.

$\Leftarrow$: Suppose $\bDia{i}{c}\top\in s$. By the existence lemma for $\bDia{i}{c}$, there exist $t$, $t'$ such that $sR_i^c tt'$. Therefore $\M,s\Vdash\bDia{i}{c}\top$.
\end{proof}

\begin{theorem}
$\SMLKv$ is strongly complete w.r.t.\ $\MLKv$ models.
\end{theorem}
A corollary follows immediately based on Proposition~\ref{prop.elkv}: 
\begin{corollary}
$\SMLKv$ (viewing \texttt{INCLT} as $\neg \Kv_i\top\to\hK_i\top$) is strongly complete w.r.t.\ $\ELKvr$ models.
\end{corollary}

\section{Discussion and future work}
In this paper, we introduce a ternary relation based simple semantics to the ``knowing value'' logic without explicit value assignments. Under this semantics, the logic can be viewed as a disguised normal modal logic with both standard unary and binary modalities. The use of this perspective is demonstrated by various applications. 

Another intuitive way to simplify the original FO-Kripke  semantics is to introduce a binary relation $\asymp_c$ for each $c$ representing the inequality of the value of $c$. Correspondingly, in the language, besides $\Diamond_i\phi$ we may introduce $\Diamond_c\phi$ formulas saying that there is a different world where $\phi$ holds but $c$ has a different value compared to the current world. However, it is not straightforward to express $\neg \Kv(\psi, c)$ in this language. The closest  counterpart  $\Diamond_i(\psi\land \Diamond_c\psi)$ will not do the job alone. We probably need to add a further condition: $\Diamond_i\Diamond_c p\to \Diamond_ip$ which says the $\asymp_c$ successors of an $i$-reachable world are again $i$-reachable. Actually it means that we should combine $\to_i$ and $\asymp_c$ which is almost our ternary $R_i^c$.  Moreover, to axiomatize this $\asymp_c$ we need the axioms of anti-equivalence (irreflexivity, symmetry, and anti-euclidean property\footnote{Here anti-euclidean property means if $x\asymp_cy$ and there is another world $z$ then $x\asymp_c z$ or $y\asymp_cz$. A simple disjoint union argument can show it is not modally definable. Thanks to Zhiguang Zhao for pointing it out.}). However, irreflexivity and anti-euclidean property for the binary $\asymp_c$ are not definable in modal logic. We probably need to do the same as in the $\Dia_i^c$ case: use $\KvDIS$ to capture the $i$-accessible anti-euclidean property to some extent. Having said the above, it is clear that our approach in this paper is more intuitive and technically natural.    

\medskip

To close, we list a few directions which we leave for  future occasions: 
\begin{itemize}
\item The corresponding results in the setting of epistemic (S5) models.
\item Characterization theorem of $\ELKvr$ (\MLKvr) within first-order modal logic via $\C$-bisimulation.
\item A decision procedure for $\ELKvr$ (\MLKvr) based on the simplified models. 
\item In similar ways, we can try to simplify the semantics for other ``knowing-X'' logics, such as knowing whether, knowing how, and so on. 
\end{itemize}

\bibliographystyle{aiml16}
\bibliography{aiml16}

\end{document}